\documentclass{article}

\usepackage[final]{neurips_2021}

\usepackage[utf8]{inputenc} % allow utf-8 input
\usepackage[T1]{fontenc}    % use 8-bit T1 fonts
\usepackage{url}            % simple URL typesetting
\usepackage{booktabs}       % professional-quality tables
\usepackage{amsfonts}       % blackboard math symbols
\usepackage{nicefrac}       % compact symbols for 1/2, etc.
\usepackage{microtype}      % microtypography
\usepackage{xcolor}         % colors
\usepackage{amsmath}
\usepackage{subfigure}
\usepackage{natbib}
\setcitestyle{numbers,square}
\usepackage{amssymb}
\usepackage{diagbox}
\usepackage{wrapfig}
\usepackage{picinpar}
\usepackage{cutwin}
\usepackage{bm}
\usepackage{bbm}
\usepackage{mathabx}
\usepackage{mathrsfs}
\usepackage{amsthm}
\usepackage{algorithm}
\usepackage{algorithmic}
\usepackage[algo2e]{algorithm2e}
\newcommand*\samethanks[1][\value{footnote}]{\footnotemark[#1]}
\newtheorem{definition}{Definition}
\newtheorem{problem}{Problem}
\newtheorem{theorem}{Theorem}
\newtheorem{lemma}{Lemma}
\newtheorem{remark}{Remark}
\usepackage{graphicx}
\definecolor{mydarkred}{rgb}{0.6,0,0}
\definecolor{mydarkgreen}{rgb}{0,0.6,0}
\definecolor{mydarkblue}{rgb}{0,0,0.6}
\usepackage[colorlinks,
linkcolor=mydarkred,
citecolor=mydarkgreen]{hyperref}
\usepackage{authblk}
\usepackage{multirow}
\usepackage{caption}

\title{TOHAN: A One-step Approach towards Few-shot Hypothesis Adaptation}

\author{
\textbf{Haoang Chi}$^{1,2,6}$\thanks{Equal contribution. Work done when Haoang Chi remotely visited HKBU.},~~\textbf{Feng Liu}$^{3*}$,~~\textbf{Wenjing Yang}$^{1}$\thanks{Corresponding author.},~~ \textbf{Long Lan}$^{1,6}$\samethanks[2],~~\textbf{Tongliang Liu}$^{4}$,\\\vspace{-2.5mm}
\textbf{Bo Han}$^{2}$,~~\textbf{William K. Cheung}$^{2}$,~~\textbf{James T. Kwok}$^{5}$\\
$^1$ State Key Laboratory of High Performance Computing, College of CS, NUDT\\
$^2$ CS Department, HKBU\\
$^3$ DeSI Lab, AAII, Faculty of Engineering and IT, UTS\\
$^4$ TML Lab, School of CS, Faculty of Engineering, USYD\\
$^5$ CSE Department, HKUST\\
$^6$ Peng Cheng Laboratory, Shenzhen\\
\texttt{haoangchi618@gmail.com, feng.liu@uts.edu.au, \{wenjing.yang, long.lan\}@nudt.edu.cn},~~
\texttt{\{bhanml, william\}@comp.hkbu.edu.hk}\\
\texttt{ jamesk@cse.ust.hk}
}

\begin{document}

\maketitle

\begin{abstract}
In \emph{few-shot domain adaptation} (FDA), classifiers for the {target domain} are trained with \emph{accessible} labeled data in the \emph{source domain} (SD) and few labeled data in the \emph{target domain} (TD). However, data usually contain private information in the current era, e.g., data distributed on personal phones. Thus, the private data will be leaked if we directly access data in SD to train a target-domain classifier (required by FDA methods). In this paper, to prevent privacy leakage in SD, we consider a very challenging problem setting, where the classifier for the TD has to be trained using few labeled target data and a well-trained SD classifier, named \emph{few-shot hypothesis adaptation} (FHA). In FHA, we cannot access data in SD, as a result, the private information in SD will be protected well. 
To this end, we propose a \emph{target-oriented hypothesis adaptation network} (TOHAN) to solve the FHA problem, where we generate highly-compatible unlabeled data (i.e., an intermediate domain) to help train a target-domain classifier.
% TOHAN is motivated by semi-supervised learning theory, i.e., high-compatible unlabeled data can help train a good classifier. 
TOHAN maintains two deep networks simultaneously, in which one focuses on learning an intermediate domain and the other takes care of the intermediate-to-target distributional adaptation and the target-risk minimization. 
% TOHAN is inspired by the cognitive process of human where two concepts can be connected by learning intermediate concepts gradually. 
Experimental results show that TOHAN outperforms competitive baselines significantly.
\end{abstract}

\section{Introduction}

In \emph{domain adaptation} (DA) \cite{jiahua2020what,DBLP:conf/aaai/JingLDWDSW20,DBLP:conf/nips/SongCWSS19,DBLP:conf/cvpr/SongCYWSMS20,DBLP:conf/icml/WeiZHY18}, we aim to train a target-domain classifier with data in source and target domains. Based on the availability of data in the target domain (e.g., fully-labeled data, partially-labeled data and unlabeled data), DA is divided into three categories: \emph{supervised DA} (SDA) \cite{sukhija2016}, semi-supervised DA \cite{DBLP:conf/ijcai/JiangWHSQL20} and \emph{unsupervised DA} (UDA) \cite{yiyang2020clarinet}. Since SDA methods outperform UDA methods for the same quantity of target data \cite{motiian2017few}, it becomes attractive if we can train a good target-domain classifier using labeled source data and few labeled target data \cite{DBLP:conf/icml/TeshimaSS20}.

\begin{figure}[!tp]
    \centering
    \includegraphics[width=0.6\textwidth]{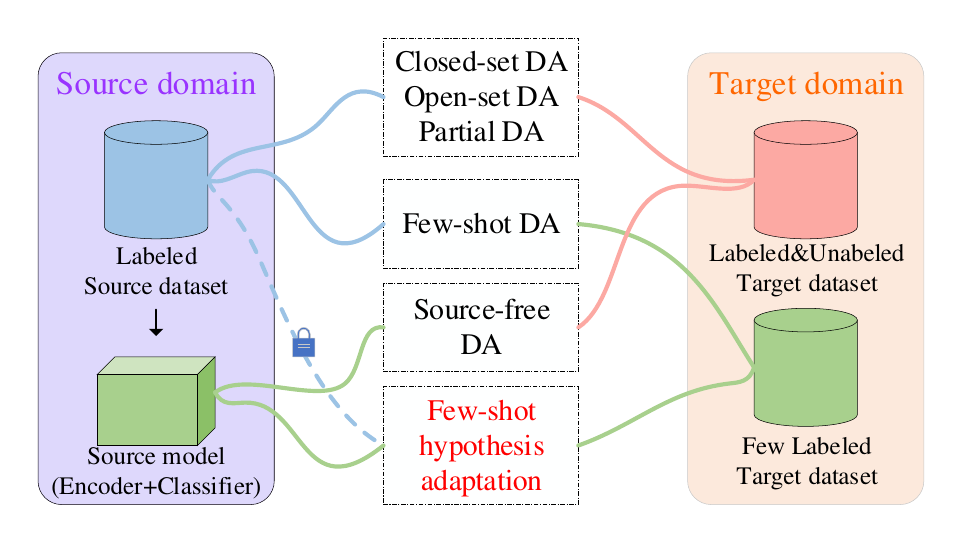}
	\caption{The \emph{few-shot hypothesis adaptation} (FHA) and existing domain adaptation problem settings. In FHA, we aim to train a classifier for the target domain only using few labeled target  data  and a  well-trained  source-domain classifier. Namely, we do not access any source data when training the target-domain classifier. This setting prevents the data leakage of the source domain passively. The lock means we cannot access data in the source domain.}
	\label{fig: setting}
% 	\vspace{-2em}
\end{figure}

%Few-shot hypothesis adaptation. For general domain adaptation (including closed-set DA, open-set DA and partial DA), we need labeled source data and enough labeled or unlabeled target data to train a target-domain classifier. In some cases, collecting enough target data is difficult. Hence we have to training the target-domain classifier with few target data, i.e. few-shot DA.

Hence, \emph{few-shot domain adaptation} (FDA) methods \cite{motiian2017few} are proposed to train a target-domain classifier with \emph{accessible} labeled data from the {source domain} and few labeled data from the target domain. Compared to SDA and UDA methods, FDA methods only require few data in the target domain, which is suitable to solve many problems, e.g., medical image processing \cite{wang2020focalmix}.
Existing FDA methods involve many approaches and  applications. Structural casual model \cite{DBLP:conf/icml/TeshimaSS20} has been proposed to overcome the problem caused by apparent distribution discrapancy. Since deep neural networks tend to overfit the few-labeled data in the training process, a meta-learning method becomes an effective solution to the FDA problem \cite{sun2019meta}. Besides, FDA methods perform well in face generation \cite{DBLP:conf/cvpr/YangL20} and virtual-to-real scene parsing \citep{zhang2019few}.
% introduce FDA

However, it is risky to directly access source data for training a target-domain classifier (required by FDA methods) due to the private information contained in the source domain. In the current era, labeled data are distributed over different physical devices and usually contain private information, e.g., data on personal phones or from surveillance cameras \cite{DBLP:conf/icml/LiangHF20}. 
Since FDA methods \cite{DBLP:conf/icml/TeshimaSS20} require abundant labeled source data to train a target-domain classifier, they may leak private information in the training process, which may result in massive loss \cite{DBLP:conf/uss/Jayaraman019}. 

In this paper, to prevent the private data leakage of the source domain in existing FDA methods, we propose a novel and very challenging problem setting, where the classifier for the target domain has to be trained using few labeled target data and a well-trained source-domain classifier, named \emph{few-shot hypothesis adaptation} (FHA, see Figure~\ref{fig: setting}). In the literature \cite{DBLP:conf/icml/LiangHF20}, researchers have adapted a source-domain hypothesis to be a target-domain classifier when abundant unlabeled target data are available. However, since these methods require abundant target data, they cannot address the FHA problem well, which has been empirically verified in Table~\ref{digits_results} and Table~\ref{tab:object}.

% \cite{liang2020we} conducts unsupervised source hypothesis transfer by freezing the classifier module (hypothesis) of the source classifier and learning the target-specific feature extraction module. \cite{kundu2020universal} proposes a two-stage learning process to address universal source-free domain adaptation.

The key benefit of FHA is that we do not need to access the source data, which wisely avoids private-information leakage of source domain under mild assumptions. Besides, since the size of datasets of most domains is large in the real world, existing FDA methods will take a long time to train a target-domain classifier. However, in FHA, we train a target-domain classifier only with a source classifier and few labeled target data, reducing the computation cost greatly. 

To address FHA, we first revisit the theory related to learning from few labeled data and try to find out if FHA can be addressed in principle. Fortunately, we find that, in \emph{semi-supervised learning} (SSL) where only few labeled data available, researchers have already shown that, a good classifier can be learned if we have abundant unlabeled data that are compatible with the labeled data. Thus, motivated by the SSL, we aim to address FHA via gradually generating highly compatible data for the target domain.
To this end, we propose a \emph{target-oriented hypothesis adaptation network} (TOHAN) to solve the FHA problem. TOHAN maintains two deep networks simultaneously, in which one focuses on learning an intermediate domain (i.e., learning compatible data) and the other takes care of the intermediate-to-target distributional adaptation (Figure~\ref{fig: model}). 
% The intermediate domain will simultaneously fit the given source-domain classifier and gradually get close to the target domain. 
% The distributional discrepancy between the intermediate domain and the target domain will be minimized. 
% The motivation behind our TOHAN is inspired by the learnability of \emph{semi-supervised learning} (SSL) that also assumes only few labeled data available. However, in SSL theory, they find that     
% human’s ``transitivity'' learning
% and inference ability \cite{bryant1971transitive}. Namely, people transfer knowledge between two related concepts via intermediate concepts as a bridge. Therefore, we can design FHA algorithms following the ideas of the two types of human learning methods.

Specifically, due to the scarcity of target data, we cannot directly generate compatible data for the target domain. Thus, we first generate an intermediate domain where data are compatible with the given source classifier and the few labeled target data. Then, we conduct the intermediate-to-target distributional adaptation to make the generated intermediate domain close to the target domain. Eventually, we embed the above procedures into our one-step solution, TOHAN, to enable gradual generation of an intermediate domain that contains highly compatible data for the target domain. According to the learnability of SSL, with the generated ``target-like'' intermediate domain, TOHAN can learn a good target-domain classifier.
% intermediate domain After a certain steps of learning an intermediate domain and the intermediate-to-target distributional adaptation, we can finally generate .

We conduct experiments on $8$ FHA tasks on $5$ datasets (\emph{MNIST}, \emph{SVHN}, \emph{USPS}, \emph{CIFAR}-$10$ and \emph{STL}-$10$). We compare TOHAN with $5$ competitive baselines. Experiments show that TOHAN effectively transfers knowledge of the source hypothesis to train a target-domain classifier when we only have few labeled target data. In other words, our paper opens a new door to the domain adaptation field, which solves private-data leakage and data shortage simultaneously.

% \begin{figure}[!tp]
%     \centering
%     \includegraphics[width=0.7\textwidth]{images/intro.pdf}
%     \vspace{-0.5em}
% 	\caption{A straightforward solution for FHA is a two-step approach. Namely, we can first generate source data and then train a target-domain classifier using the generated source data and an FDA method (e.g., \emph{few-shot  adversarial  domain  adaptation} (FADA)). 
% % 	A visualized comparison between S+FADA and TOHAN (take MNIST$\rightarrow$USPS as an example) is displayed in subfigures (a) and (b). On the left side, 
% 	Subfigure (a) illustrates source-domain data generated by a two-step method: S+FADA. It is clear that the generated data are just noise and do not contain useful information about the source domain. In subfigure (b), we illustrate the intermediate-domain data generated by our method (i.e., TOHAN). It is clear that the generated intermediate-domain data contain useful information about two domains. In subfigure (c), the histogram shows classification accuracy of the two methods, and TOHAN outperforms S+FADA clearly.}
% 	\label{fig:intro}
% 	\vspace{-1em}
% \end{figure}

\section{Few-shot Hypothesis Adaptation}
In this section, we formalize a novel and challenging problem setting, called \emph{few-shot hypothesis adaptation} (FHA).
% \subsection{Problem Setting}
Let $\mathcal{X}\subset\mathbb{R}^{d}$ be a feature (input) space and $\mathcal{Y}:=\{1,\dots,N\}$ be a label (output) space, and $N$ is the number of classes. A domain \cite{zhen2021open} for the FHA problem is defined as follows.

\begin{figure*}[!tp]
    \centering
    \includegraphics[width=1.0\textwidth]{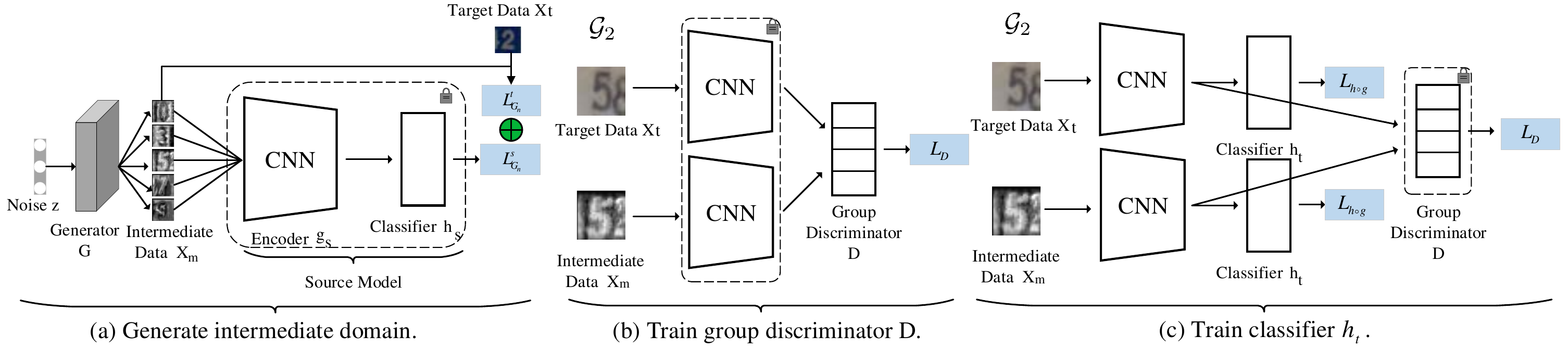}
	\caption{Overview of \emph{target-oriented hypothesis adaptation network} (TOHAN). It consists of generator \emph{G}, encoder \emph{$g_s$}, \emph{$g_t$} (initialize \emph{$g_t$}=\emph{$g_s$}), classifier \emph{$h_s$}, \emph{$h_t$} (initialize \emph{$h_t$}=\emph{$h_s$}) and group discriminator \emph{D}. (a) Firstly, we train a generator \emph{G} using the source classifier \emph{$g_s$}, \emph{$h_s$} and target data \emph{$D_t$}. Then we generate intermediate data between the two domains. (b) We freeze \emph{$g_t$} and \emph{$h_t$} and update group discriminator \emph{D}. (c) We freeze \emph{D} and update \emph{$g_t$} and \emph{$h_t$}. In subfigures (b) and (c), they show a data pair from $\mathcal{G}_2$, where the two data points come from the same class but different domains.}
	\label{fig: model}
% 	\vspace{-1em}
\end{figure*}

\begin{definition}[Domains for FHA]
Given random variables $X_s,X_t\in \mathcal{X}$, $Y_s,Y_t\in \mathcal{Y}$, the source and target domains are joint distributions $P\left(X_s,Y_s\right)$ and $P\left(X_t,Y_t\right)$, respectively, where the joint distributions $P\left(X_s,Y_s\right)\,\neq\,P\left(X_t,Y_t\right)$ and $\mathcal{X}$ is compact.
\end{definition}

Then the FHA problem is defined as follows.

\begin{problem}[FHA]
Given a model (consisting of an encoder $g_s$ and a classifier $h_s$) trained on the source domain $P\left(X_s,Y_s\right)$ and independent and identically distributed (i.i.d.) labeled data $D_{t}=\left\{\left(x_{t}^{i},y_{t}^{i}\right)\right\}_{i=1}^{n_t}$ ($n_t\le 7N$, following \cite{Park_2019_ICCV}) drawn from the target domain $P\left(X_t,Y_t\right)$, the aim of FHA is to train a classifier $\emph{$h_t$}:\mathcal{X}\to\mathcal{Y}$ with $g_s$, $h_s$ and $D_t$ such that $h_t$ can accurately classify target data drawn from $P\left(X_t,Y_t\right)$. 
\end{problem}

\begin{remark}\upshape
\label{rem:assum}
In FHA, there exists an assumption: malicious attackers cannot easily find source-domain-like data from the Internet and via some other ways. Otherwise, attackers may use the attack methods \cite{DBLP:conf/cvpr/ZhangJP0LS20} to recover the training data, leading to data leakage.
\end{remark}

\textbf{Possible Privacy-leakage Issues in FHA.}
The assumption in Remark~\ref{rem:assum} is derived from the attack methods that aim to recover training data from a well-trained model. According to recent model-inversion attack methods \cite{DBLP:conf/cvpr/ZhangJP0LS20}, they need to access auxiliary data whose background is similar to the training data to help recover input data. There also exists a white-box inference attack method \cite{DBLP:conf/sp/NasrSH19} that determines a data point's membership in the training set of the model. Therefore, FHA belongs to passive protection, requiring the training data of source model are sufficiently different from public data. To thoroughly avoid this issue, data owners might utilize the defending techniques (against the model-inversion attacks) to train their source models.

\textbf{Comparison with Few-shot Learning.} 
The main difference between FHA and \emph{few-shot learning} (FSL) is the representation of source domain. For FHA, source domain is represented by a model trained with source data. While, for FSL, source domain is represented by labeled data themselves \cite{liu2019learning,liu2021a}. Besides, the data used to train source classifiers come from different domains from target data in FHA, while source data and target data come from the same domain in FSL. The works \cite{DBLP:conf/iccv/HariharanG17,DBLP:conf/cvpr/WangGHH18} propose to hallucinate additional training examples to solve few-shot visual recognition, inspired by human's visual imagination. Meta-learning \cite{DBLP:conf/icml/FinnAL17,DBLP:conf/nips/SnellSZ17} also performs well in FSL by learning the distribution of tasks with high generalization ability. As using few data for training easily leads to overfitting, there are works \cite{DBLP:journals/pami/Fei-FeiFP06,DBLP:conf/eccv/ZhangTJ18} trying to constrain the hypothesis space to avoid it. \emph{Data augmentation generative adversarial network} (DAGAN) \cite{antoniou2017data} aims to augment target data through a conditional generative adversarial network to enhance the few-shot learning procedure.
%The main difference between FHA and FSL is the prior knowledge \citep{DBLP:journals/csur/WangYKN20}. The prior knowledge of FSL mainly includes various types of numerical information and comes from the same distribution with their tasks \citep{DBLP:journals/csur/WangYKN20}. For example, \cite{DBLP:journals/pami/0002X00LTG19} uses the data itself as prior knowledge, and \cite{DBLP:conf/nips/SnellSZ17} uses the pairwise similarity, which is relatively weaker than the former. In addition, \emph{model-agnostic meta learning} (MAML) requires data to optimize a meta-learner as prior knowledge \citep{DBLP:conf/icml/FinnAL17}. However, the prior knowledge of FHA is just a well-trained classifier and training data of this classifier come from different distribution with $P(X_t,Y_t)$.
%(e.g. similarity \citep{DBLP:conf/nips/SnellSZ17}, learning \citep{DBLP:conf/icml/FinnAL17} or data \citep{DBLP:journals/pami/0002X00LTG19}), however, the prior knowledge of FHA is only a well-trained source classifier, which is more hard to leverage.

\textbf{Comparison with UDA.} The main differences between FHA and UDA lie in the amount and label of data in the two domains. For the source domain, UDA requires a large amount of labeled data \cite{jiahua2019semantic,li2021how}, while FHA only requires a well-trained model. For the target domain, UDA requires a large amount of unlabeled data \cite{sun2021what,li2021bridging}, while FHA requires \emph{few} labeled data.
% \paragraph{Possible solutions to FHA.} As mentioned in Figure~\ref{fig:intro}, a straightforward solution to the FHA problem is a two-step approach. We can train a source-data generator under the guidance of source hypothesis, and then we use it to generate source data. Finally, we can apply the restored source data into the FADA method to train a target domain classifier.

\textbf{Comparison with FDA.}
With the development of FSL, researchers also apply ideas of FSL into domain adaptation, called \emph{few-shot domain adaptation} (FDA). 
% With limited labeled target data, there are no guarantee of the performance of target-domain classifier. The existing works about FDA has some solution to this problem. 
FADA \citep{motiian2017few} is a representative FDA method, which pairs data from the source domain and data from the target domain and then follows the adversarial domain adaptation method. Casual mechanism transfer \citep{DBLP:conf/icml/TeshimaSS20} is another novel FDA method dealing with a meta-distributional scenario, in which the data generating mechanism is invariant among domains. Nevertheless, FDA methods still need to access many labeled source data for training, which may cause the private-information leakage of the source domain. 
% We would like to see whether we can effectively finish FDA without source data.

\textbf{Comparison with Hypothesis Transfer Learning.}
In \emph{hypothesis transfer learning} (HTL), we can only access a well-trained source-domain classifier and small labeled or abundant unlabeled target data. \cite{DBLP:conf/icml/KuzborskijO13} requires small labeled target data and uses the Leave-One-Out error to find the optimal transfer parameters. Later, SHOT \cite{DBLP:conf/icml/LiangHF20} is proposed to solve the HTL with many unlabeled target data by freezing the source-domain classifier and learning a target-specific feature extraction module. 
% \cite{hou2020source} proposes an image translation method that transfers the style of target images to that of unseen source images. 
As for the universal setting, a two-stage learning process \citep{kundu2020universal} has been proposed to address the HTL problem. Compared with FHA, HTL still requires at least small target data (e.g., at least $12$ samples in binary classification problem \cite{DBLP:conf/icml/KuzborskijO13}, or at least $10\%$ target data are labeled \cite{DBLP:conf/cvpr/AhmedLPR20}). In FHA, we focus on a more challenging situation: only few data (e.g., one sample per class) are available. Besides, previous solutions to HTL mainly focus on mortifying existing hypotheses or loss functions used for fine-tuning. However, our solution stems from the learnability of semi-supervised learning (Section~\ref{sec:SSL_moti}) and try to generate more compatible data, which is quite different from previous works. %In this paper, we modify the newest HTL method, SHOT \cite{DBLP:conf/icml/LiangHF20}, as one of our baselines. The modified SHOT can leverage labeled target data to train a good target-domain classifier.

\section{How to Learn from Few-shot Data in Principle}\label{sec:SSL_moti}
From the view of statistical learning theory \cite{book:680}, it is unrealistic to directly learn an accurate target-domain classifier only with few labeled data. However, the amount of labeled data in \emph{semi-supervised learning} (SSL) \cite{DBLP:reference/ml/Zhu10} is also few (e.g., one sample per class), but SSL methods still achieves good performance across various learning tasks, which motivates us to consider solving FHA in the view of SSL. 
First, we will show theoretical analysis regarding learnability of SSL.

\paragraph{Learnability of SSL.} For simplicity, we consider the $0$-$1$ semi-supervised classification problem. Let $c^*:\mathcal{X}\to\{0,1\}$ be the optimal target classifier and $\mathcal{H}=\{h:\mathcal{X}\to\{0,1\}\}$ is a hypothesis space. 
% which may not be in hypothesis space $\mathcal{H}$. 
Let $err(h)=\mathbb{E}_{x\sim P}[h(x)\neq c^{*}(x)]$ be the true error rate of a hypothesis $h$ over a distribution $P$. In SSL, its learnability mainly depends on the compatibility $\chi:\mathcal{H}\times\mathcal{X}\mapsto[0,1]$ that measures how ``compatible'' $h$ is to an unlabeled data $x$. Let $\chi(h,P)=\mathbb{E}_{x\sim P}[\chi(h,x)]$ be the expectation of compatibility of data from $P$ on a classifier $h$. If the unlabeled data and $c^*$ are highly compatible (i.e., $\chi(c^*,P)$ closes to $1$), then, in theory, we can learn a good classifier with few labeled data and sufficient unlabeled data. Specifically, we have the following theorem (see proof in Appendix~\ref{Asec:Thm}). 
\begin{theorem}
\label{thm:1}
Let $\hat{\chi}(h,S)=\frac{1}{|S|}\sum_{x\in S}\chi(h,x)$ be the empirical compatibility over unlabeled dataset $S$. Let $\mathcal{H}_0 = \{h\in\mathcal{H}:\widehat{err}(h)=0\}$. If $c^*\in\mathcal{H}$ and $\chi(c^*,P)=1-t$, then $m_u$ unlabeled data and $m_l$ labeled data are sufficient to learn to error $\epsilon$ with probability $1-\delta$, for
\begin{equation}
    m_u=\mathcal{O}\left(\frac{VCdim(\chi(\mathcal{H}))}{\epsilon^2}\log\frac{1}{\epsilon}+\frac{1}{\epsilon^2}\log\frac{2}{\delta}\right)
\end{equation}
and
\begin{equation}
    m_l=\frac{2}{\epsilon}\left[\ln(2\mathcal{H}_{P,\chi}(t+2\epsilon)[2m_l,P])+\ln\frac{4}{\delta}\right],
\end{equation}
where $\chi(\mathcal{H})=\{\chi_h:h\in\mathcal{H}\}$, $\chi_h(\cdot)=\chi(h,\cdot)$, and $\mathcal{H}_{P,\chi}(t+2\epsilon)[2m_l,P]$ is the expected number of splits of $2m_l$ data drawn from $P$ using hypotheses in $\mathcal{H}$ of compatibility more than $1-t-2\epsilon$. In particular, with probability at least $1-\delta$, we have $err(\hat{h})\le\epsilon$,
where
\begin{equation}
    \hat{h}=\mathop{\arg\max}_{h \in \mathcal{H}_0}\hat{\chi}(h,S).
\end{equation}
\iffalse
% the $h\in\mathcal{H}$ that optimizes $\widehat{err}(h)$ subject to $\widehat{err}_{unl}(h,S)\le t+\epsilon$ has
\begin{equation}
\label{ssl_bound}
    err(\hat{h}_{t+\epsilon})\le err(h_t^*)+\epsilon+\sqrt{\frac{\log(4/\delta)}{2m_l}}\le err(h_t^*)+2\epsilon,
\end{equation}
where 
\begin{equation}
\label{h_hat_def}
    \hat{h}_{t+\epsilon}=\mathop{\arg\min}_{h\in\mathcal{H}_{t+\epsilon}}\widehat{err}(h),
\end{equation}
$\chi(\mathcal{H})=\{\chi_h:h\in\mathcal{H}\}$, $\chi_{h}(\cdot)=\chi(h,\cdot)$, and $\mathcal{H}_{D,\chi}(\tau)=\{h\in\mathcal{H}:err_{unl}(h)\le\tau\}$. And $\mathcal{H}_{D,\chi}(t+2\epsilon)[2m_l,D]$ is the expected number of splits of $2m_l$ data drawn from $D$ using hypothesis in $\mathcal{H}$ of unlabeled error rate less than $t+2\epsilon$.
\fi
\end{theorem}

%Given a value $t$, Theorem~\ref{thm:1} bounds the number of labeled data needed to achieve error at most $\epsilon$ larger than that of the best function $h_t^*$ of unlabeled error rate at most $t$. 
\begin{remark}\upshape
If the unlabeled data are highly compatible to $c^*$, $t$ is small, which results in a smaller $m_l$. Namely, with the smaller $m_l$, we can still achieve a low error rate. In view of Theorem \ref{thm:1}, it is clear that SSL will be learnable if many compatible unlabeled data are available. Motivated by SSL, we wonder if we can generate compatible data to help our learning task. The answer is \emph{affirmative}. 
\end{remark} 

\paragraph{Solving FHA in Principle.} Motivated by Theorem~\ref{thm:1}, finding many highly compatible unlabeled data is a breakthrough point for FHA. Hence, generating unlabeled target data is a straightforward solution. However, due to the shortage of existing target data, directly generating them is unrealistic. To solve this problem, we can ask for help from the source classifier. In our paper, we first try to generate intermediate domain $P_m$ containing knowledge of source and target domains, which are compatible with both the source classifier and target classifier, i.e.,
\begin{align}
\label{eq:55}
    P_m=\mathop{\arg\max}\limits_{P}[\chi(h_s,P)+\chi(h_t,P)],
\end{align}
where $\chi(h_s,P)$ (resp. $\chi(h_t,P)$) measures how compatible $h_s$ (resp. $h_t$) is with the data distribution $P$. Then, we will adapt intermediate domain $P_m$ to the target domain via distributional adaptation with the training procedure going on. Finally, we can obtain many unlabeled data that are compatible with $h_s$ and $h_t$ (more compatible with $h_t$), meaning that, based on Theorem~\ref{thm:1}, we can address FHA in principle. According to Eq.~\eqref{eq:55}, it can be seen that we can have two straightforward solutions: maximizing $\chi(h_s,P)$ or $\chi(h_t,P)$, corresponding to S+FADA and T+FADA in benchmark solutions. The results in Table~\ref{digits_results} and Table~\ref{tab:object} indicate that these two straightforward solutions cannot address FHA well, which motivates us to maximize them simultaneously, which is realized below.
% Specifically, intermediate data can be classified accurately by source classifier, and they are more oriented to target domain with the training procedure going on. 
% To realize this, we introduce our solution in the next section.

\section{Target-Oriented Hypothesis Adaptation Network for FHA Problem}
%Although S+FADA (see Figure~\ref{fig:intro}) could be a solution to the FHA problem, the generated data in the two-step approach may have large distributional discrepancy with target domain (see Figure~\ref{fig:intro}a). This will make adaptation from the generated data to the target domain much difficult (see Figure~\ref{fig:intro}c). To solve the FHA problem well, we propose a powerful one-step method: \emph{target orientated hypothesis adaptation network} (TOHAN, see Figure~\ref{fig: model}). 

This section presents a powerful one-step approach: \emph{target-oriented hypothesis adaptation network} (TOHAN, see Figure~\ref{fig: model}). TOHAN can generate data that are highly compatible with both the source classifier and target classifier and adapt the knowledge of these data to the target domain gradually.

% TOHAN is inspired by human’s ``transitivity'' learning
% and inference ability \citep{bryant1971transitive}. Namely, people transfer knowledge between two related concepts via one or more intermediate concepts as a bridge. Thus, TOHAN will generate intermediate data (as a bridge) and adapt the generated data to the target domain simultaneously (continuously learning from the source-domain classifier via the bridge). In this way, the generated data will always be guided by the target domain and will not have a large distributional discrepancy with the target domain. In the following, we will introduce the intermediate domain generation and intermediate-to-target distributional adaptation separately, and finally show the whole algorithm of TOHAN.

% TOHAN mainly consists of $N$ generators $\{G_n\}_{n=1}^{N}$ ($N$ is the number of classes of domains), $2$ encoders \emph{$g_s$}, \emph{$g_t$}, $2$ classifiers \emph{$h_s$}, \emph{$h_t$} and one domain discriminator \emph{D}. We 
% will introduce main procedures of TOHAN in the following.

\paragraph{Intermediate domain generation.}
The first step of TOHAN is to generate the intermediate domain data (Figure~\ref{fig: model}a).  We input Gaussian random noise \emph{z} to a generator $G_n$ (taking the $n^{th}$ class for an example), then the generator outputs generated data. We aim to generate data satisfying (1) the generated data $G_n(z)$ can be correctly classified by the given source classifier \emph{$f_s=h_s\circ g_s$}, and (2) $G_n(z)$ becomes closer to the target domain with training procedure going on. Thus, there are two loss functions regarding the intermediate domain generation. The first one is as follows.

% Please note that generators used in TOHAN is different from GAN \cite{goodfellow2014generative}.

%Since we aim to generate data belonging to the $n^{th}$ class, we will obtain this probability when $G_n(z)$ is inputted to source classifier \emph{$f_s$}, i.e. the $n^{th}$ item of $f_s\left(G_{n}\left(z\right)\right)$, denoted as $l_n$. The closer $l_n$'s elements tend to 1, the more $G_n(z)$ looks like a data belonging to class n. In this case, part of loss of $G_n$ is standard mean square error defined as follows:
Without loss of the generality, we assume $G_n(z)$ generates $B$ images, where $B$ is the batchsize in the training process of TOHAN. When $G_n(z)$ is inputted to the source-domain classifier \emph{$f_s$}, we will obtain an $B\times N$ matrix $\mathbf{G}_n^M$, where the $i^{th}$ row in $\mathbf{G}_n^M$ represents probability of the $i^{th}$ generated image belonging to each class. Thus, the $n^{th}$ column in $\mathbf{G}_n^M$ represents the probability that the $B$ generated images belongs to the $n^{th}$ class, and we denote the $n^{th}$ column in $\mathbf{G}_n^M$ as $l_n$. Since $G_n(z)$ aims to generate data belonging to the $n^{th}$ class, we should update parameters of $\mathbf{G}_n^M$ to make each element in $l_n$ close to $1$. 
Namely, the first loss function to train the $G_n$ can be defined as
\begin{equation}
    \mathcal{L}_{G_n}^{s}=\frac{1}{B}\left\|l_{n}-\mathbbm{1}\right\|_{2}^{2},
    \label{eq:1}
\end{equation}
where $\mathbbm{1}$ is a $B$-by-$1$ vector whose elements are $1$.

As discussed before, we also want to reduce the distance between the generated data $G_n(z)$ and the target data whose labels are $n$. In this way, we can make the generated data close to the target domain and attain an intermediate domain $P_m$. Following \cite{lin2017focal}, we adopt an augmented $L_1$ distance $\|X-Y\|_1=\sum_i\omega_{i}\left|X_{i}-Y_{i}\right|$, where $\omega_{i}=\left|X_{i}-Y_{i}\right|^{2}/\|X-Y\|_{2}$. Compared to the ordinary $\ell_1$ norm, the augmented $L_1$ distance encourages larger gradients for feature dimensions with higher residual error \citep{lin2017focal}. 
Compared to the $\ell_2$ norm, since $L_1$ distance is more robust to outliers \citep{DBLP:conf/nips/NieHCD10}, it is better to measure the distance between generated images and target images. Thus, the second loss to train $G_n$ is defined as follows,
\begin{equation}
    \mathcal{L}_{G_n}^{t}=\frac{1}{MBK}\sum_{i=1}^{B}\sum_{k=1}^{K}\left\|x_{m}^{i}-x_{t}^{k}\right\|_1,
    \label{eq:2}
\end{equation}
where $M=\max\limits_{x_1,x_2\in\mathcal{X}}\|x_1-x_2\|_1$ ($\mathcal{X}$ is compact and $\|\cdot\|_1$ is continuous) and $G_n(z):=\{x_{m}^{i}\}_{i=1}^{B}$. Combining Eq.~\eqref{eq:1} and Eq.~\eqref{eq:2}, we obtain the total loss to train the generator $G_n$:
\begin{equation}
\begin{aligned}
    \mathcal{L}_{G_n}&=\mathcal{L}_{G_n}^{s}+\lambda\mathcal{L}_{G_n}^{t}=\frac{1}{B}\left\|l_{n}-\mathbbm{1}\right\|_{2}^{2}+\frac{\lambda}{MBK}\sum_{i=1}^{B}\sum_{k=1}^{K}\left\|x_{m}^{i}-x_{t}^{k}\right\|_1,
    \label{eq:3}
\end{aligned}
\end{equation}
where $\lambda$ is a hyper-parameter between two losses to tradeoff the weight of knowledge of source and target domains. 
To ensure that the generated data are high-quality images, we train the generator $G_n$ ($n=1,\dots,N$) for some steps all alone. Note that, Eq.~\eqref{eq:3} corresponds to Eq.~\eqref{eq:55}, and Eq.~\eqref{eq:1} (resp. Eq.~\eqref{eq:2}) is corresponding to $\chi(h_s,P_m)$ (resp. $\chi(h_t,P_m)$). Then we conduct intermediate-to-target distributional adaptation (see the next paragraph) and generation simultaneously.

\paragraph{Intermediate-to-target distributional adaptation.}
% Now we have only few target data and source hypothesis composed by encoder \emph{$g_s$} and classifier \emph{$h_s$}, and FHA problem is to find the best approximation for \emph{$g_t$} and \emph{$h_t$}. If  \emph{$g_s$} and \emph{$g_t$} can embed source and target data to a domain invariant space we can use source classifier \emph{$h_s$} for target data.

Now, we focus on how to construct \emph{domain-invariant representations} (DIP) between the intermediate domain and the target domain. Through DIP, a classifier for the intermediate domain can be used to classify target data well \cite{feng2020learning,li2021how}. %\citep{Gong2016CTC}.

Since we only have few target data per class, we aim to ``augment'' them. Following \cite{motiian2017few}, we can overcome the shortage of target data by pairing them with the corresponding intermediate data. Specifically, we create $4$ groups of data pairs: $\mathcal{G}_1$ consists of data pairs from the same domain with the same label, $\mathcal{G}_2$ consists of pairs from different domains (one from the intermediate and one from the target domain) but with the same label, $\mathcal{G}_3$ consists of pairs from the same domain with different labels, and $\mathcal{G}_4$ consists of pairs from different domains (one from the intermediate and one from the target domain) and with different labels. 
% By this, we can leverage both domain and label information of model and training data.

Based on the above four groups, we construct a four-class group discriminator \emph{D} to decide which of the four groups a given data pair belongs to, which differs from classical adversarial domain adaptation \citep{ganin2016domain,DBLP:conf/ijcai/JiangWHSQL20}. The group discriminator \emph{D} aims to classify the data pair groups. 
% to align the semantic of intermediate and target domains. 
As a classification problem, we train \emph{D} with the standard categorical cross-entropy loss:
\begin{equation}
    \mathcal{L}_{D}=-\hat{\mathbb{E}}\left[\sum_{i=1}^{4}y_{\mathcal{G}_i}\log\left(D\left(\phi\left(\mathcal{G}_{i}\right)\right)\right)\right],
    \label{eq:4}
\end{equation}
where $\hat{\mathbb{E}}[\cdot]$ represents the empirical mean value, $y_{\mathcal{G}_i}$ is the label of group $\mathcal{G}_i$, and $\phi(\mathcal{G}_i):=\left[g_{t}(x_1),g_{t}(x_2)\right]$, $(x_1,x_2)\in\mathcal{G}_i$, and $g_t$ is the encoder on target domain. 
% $\phi$ is a concatenation function combining two results of \emph{$g_t$} given a data pair as input.  
Note that we freeze \emph{$g_t$} when minimizing the above loss function (see Figure~\ref{fig: model}b).

Next, we turn to train \emph{$g_t$} and \emph{$h_t$} with the group discriminator \emph{D} fixed, which confuses \emph{D} between $\mathcal{G}_1$ and $\mathcal{G}_2$ (also $\mathcal{G}_3$ and $\mathcal{G}_4$). However, we need \emph{D} to correctly discriminate positive pairs ($\mathcal{G}_1$, $\mathcal{G}_2$) from negative pairs ($\mathcal{G}_3$, $\mathcal{G}_4$). This means that domain confusion and classification are realized at the same time. We firstly initialize \emph{$g_t$} and \emph{$h_t$} with the same weight as \emph{$g_s$} and \emph{$h_s$}, respectively. Motivated by the non-saturating game \citep{goodfellow2016nips}, we minimize the following loss to update \emph{$g_t$} and \emph{$h_t$} (see Figure~\ref{fig: model}c):
\begin{equation}
\label{eq:5}
    \mathcal{L}_{h\circ g}=-\beta \hat{\mathbb{E}}\left[y_{\mathcal{G}_1}\log\left(D\left(\phi\left(\mathcal{G}_2\right)\right)\right)-y_{\mathcal{G}_3}\log\left(D\left(\phi\left(\mathcal{G}_4\right)\right)\right)\right]
    %&+\hat{\mathbb{E}}\left[\ell\left(f_{t}\left(X_m\right),Y_m\right)\right]
    +\hat{\mathbb{E}}\left[\ell\left(f_{t}\left(X_t\right),f_t^*(X_t)\right)\right],
\end{equation}
where $\beta$ is a hyper-parameter to tradeoff confusion and classification and $\ell$ is the cross-entropy loss. $f_{t}:=g_{t}\circ h_{t}$ is the target model and $f_t^*$ is the optimal target model. Corresponding to Theorem~\ref{thm:1}, optimizing the first term in Eq.~\eqref{eq:5} increases compatibility of the target model with the intermediate data, and optimizing the second term in Eq.~\eqref{eq:5} reduces $\widehat{err}(h_t)$, resulting in a smaller $err(h_t)$. Compared to \cite{motiian2017few}, Eq.~\eqref{eq:5} means that we train the target model by confusing $D$ and improving classification accuracy simultaneously.

\begin{algorithm}[!t]
\small
\caption{Target-oriented  hypothesis  adaptation  network (TOHAN)}
\label{alg:algorithm}
\textbf{Input}: encoder \emph{$g_s$}, classifier \emph{$h_s$}, $D_{t}= \left\{x_{t}^{i},y_{t}^{i}\right\}_{i=1}^{n_{t}}$, learning rate $\gamma_{1}$, $\gamma_{2}$, $\gamma_{3}$ and $\gamma_{4}$, total epoch $T_{max}$, pretraining \emph{D} epoch $T_{d}$, adaptation epoch $T_{f}$, network parameter $\left\{\theta_{G_{n}}\right\}_{n=1}^{N}$, $\theta_{h\circ g}$, $\theta_{D}$.\\
\iffalse
{\bfseries 3: Shuffle} training set $\tilde{D}$; \hfill // Noisy dataset

\For{$N = 1,\dots,N_{max}$}{

{\bfseries 4: Fetch} mini-batch $\check{D}$ from $\tilde{D}$;

{\bfseries 5: Update} Branch-I: $F_1,F_2$ = Checking($F_1,F_2,\check{D},\eta,R(T)$);

{\bfseries 6: Fetch} mini-batch $\check{D}_t$ from $\tilde{D}_t^l$;

{\bfseries 7: Update} Branch-II: $F_{t1},F_{t2}$ = Checking($F_{t1},F_{t2},\check{D}_t,\eta,R_t(T)$);

}
\fi
%\begin{algorithmic}[1] %[1] enables line numbers
{\bfseries 1: Initialize} $\left\{\theta_{G_{n}}\right\}_{n=1}^{N}$ and $\theta_{D}$;

\For{$t=1,2,.....,T_{max}$}{

{\bfseries 2: Initialize} $\mathcal{D}_m=\varnothing$

\For{$n=0,1,\dots,N-1$}{

{\bfseries 3: Generate} random noise $z$;

{\bfseries 4: Generate} data $G_n(z)$ then \textbf{add} them to $\mathcal{D}_m$

%\STATE \textbf{Calculate} $\mathcal{L}_{G_{n}}(z,D_t)$ using Eq.~\eqref{eq；3};
{\bfseries 5: Update} $\theta_{G_{n}}\leftarrow\theta_{G_{n}}-\gamma_{1}\nabla \mathcal{L}_{G_{n}}\left(z,D_t\right)$ using Eq.~\eqref{eq:3};
}
\If{$t=T_{max}-T_{f}$}{

\For{$i=1,2,\dots,T_{d}$}{

{\bfseries 6: Sample} $\mathcal{G}_{1}$, $\mathcal{G}_{3}$ from $\mathcal{D}_{m}\times\mathcal{D}_{m}$;

{\bfseries 7: Sample} $\mathcal{G}_{2}$, $\mathcal{G}_{4}$ from $\mathcal{D}_{m}\times \mathcal{D}_{t}$;

%\STATE \textbf{Calculate} $\mathcal{L}_{D}\left(\mathcal{G}_{1},\mathcal{G}_{2},\mathcal{G}_{3},\mathcal{G}_{4}\right)$ using Eq.~\eqref{eq:4};
{\bfseries 8: Update} $\theta_{D}\leftarrow\theta_{D}-\gamma_{2}\nabla \mathcal{L}_{D}\left(\{\mathcal{G}_i\}_{i=1}^{4}\right)$ using Eq.~\eqref{eq:4};
}
}
\If{$t\ge T_{max}-T_{f}$}{

{\bfseries 9: Sample} $\mathcal{G}_{1}$, $\mathcal{G}_{3}$ from $\mathcal{D}_{m}\times\mathcal{D}_{m}$;

{\bfseries 10: Sample} $\mathcal{G}_{2}$, $\mathcal{G}_{4}$ from $\mathcal{D}_{m}\times \mathcal{D}_{t}$;

%\STATE \textbf{Calculate} $\mathcal{L}_{h\circ g}\left(\mathcal{G}_{1},\mathcal{G}_{2},\mathcal{G}_{3},\mathcal{G}_{4},x_{m},x_{t}\right)$ using Eq.~\eqref{eq:5};
{\bfseries 11: Update} $\theta_{h\circ g}\leftarrow \theta_{h\circ g}-
\gamma_{3}\mathcal{L}_{h\circ g}(\{\mathcal{G}_i\}_{i=1}^{4},x_{m},x_{t})$ using Eq.~\eqref{eq:5};
%\STATE \textbf{Calculate} $\mathcal{L}_{D}\left(\mathcal{G}_{1},\mathcal{G}_{2},\mathcal{G}_{3},\mathcal{G}_{4}\right)$ using Eq.~\eqref{eq:4};

{\bfseries 12: Update} $\theta_{D}\leftarrow\theta_{D}-\gamma_{4}\nabla \mathcal{L}_{D}\left(\{\mathcal{G}_i\}_{i=1}^{4}\right)$ using Eq.~\eqref{eq:4};
}
}
%\end{algorithmic}
\textbf{Output}: the neural network $h_{t}\circ g_{t}$.
\end{algorithm}

\paragraph{TOHAN: A one-step solution to FHA.}
Although we can sequentially combine the above two steps to solve the FHA problem (i.e., a two-step solution), the fixed intermediate domain (generated by the first step) may have large distributional discrepancy with the target domain. As a result, such two-step solution may not obtain a good target-domain classifier. To address this issue, we introduce a one-step solution TOHAN. The ablation study verifies that TOHAN outperforms such two-step solution (see ST+F and TOHAN in Table~\ref{ablation_study}).

The entire training procedure of TOHAN is shown in Algorithm~\ref{alg:algorithm}. Since the convergence speed of generator \emph{G} is relatively slow, the quality of generated data is poor at the beginning of the training process of \emph{G}. Thus, we will train the generator \emph{G} for a certain number of epochs before performing intermediate-to-target distributional adaptation (lines $2$ to $5$). When the generator can generate high-quality images, we train the generator and conduct adaptation together. 

We train every generator $G_n$ ($n=1,2,\dots,N$) separately, and we generate the intermediate domain data using the latest generators. Then, we pair the intermediate data with the target data and pre-train the group discriminator \emph{D} (lines $6$ to $8$). Next, we pair the intermediate data with target data and conduct the adaptation (lines $9$ to $12$). After conducting intermediate-to-target distributional adaptation, we obtain better \emph{$g_t$} and \emph{$h_t$}, i.e. classifying the target data more accurately. With the better target-domain classifier, we can make the generated intermediate data get closer to the target domain, in turn, these generated intermediate data further promote adaptation performance. 

\begin{table*}[!t]
\centering
\footnotesize
%\small
\setlength\tabcolsep{4.4pt}
\caption{Classification accuracy$\pm$standard deviation ($\%$) on $6$ digits FHA tasks. Bold value represents the highest accuracy on each column.}
\vspace{1mm}
% \resizebox{0.92\columnwidth}{!}{%
  \begin{tabular}{  l  c  c  c  c  c  c  c  c  c  c  c  c}
  \toprule
   \multicolumn{1}{c}{\multirow{2}{*}{Tasks}} & \multicolumn{1}{c}{\multirow{2}{*}{WA}} & 
   \multicolumn{1}{c}{FHA} & \multicolumn{7}{c}{{Number of Target Data per Class}} \\
   \cline{4-10}
&  & Methods & {1} & {2}  & {3} & {4} & {5}  & {6} & {7} \\
\midrule
 \multirow{5}{*}{\emph{M}$\rightarrow$\emph{S}}   & \multirow{5}{*}{24.1} & {FT} &  26.7$\pm$1.0 & 26.8$\pm$2.1 & 26.8$\pm$1.6 & 27.0$\pm$0.7 & 27.3$\pm$1.2 & 27.5$\pm$0.8 & 28.3$\pm$1.5 \\
 %& & {ATL} & 10.7$\pm$1.3 & 10.3$\pm$1.5 & 9.5$\pm$0.9 & 10.3$\pm$1.7 & 12.6$\pm$0.9 & 13.4$\pm$1.1 & 12.9$\pm$2.1 \\
 &  & {SHOT} & 25.7$\pm$2.2 & 26.9$\pm$1.2 & 27.9$\pm$2.6 & 29.1$\pm$0.4 & 29.1$\pm$1.4 & 29.6$\pm$1.7 & 29.8$\pm$1.5 \\
 &  & {S+F} & 25.6$\pm$1.3 & 27.7$\pm$0.5 & 27.8$\pm$0.7 & 28.2$\pm$1.3 & 28.4$\pm$1.4 & 29.0$\pm$1.0 & 29.6$\pm$1.9 \\
 &  & {T+F} & 25.3$\pm$1.0 & 26.3$\pm$0.8 & 28.9$\pm$1.0 & 29.1$\pm$1.3 & 29.2$\pm$1.3 & 31.9$\pm$0.4 & 32.4$\pm$1.8 \\
 &  & {TOHAN} & \textbf{26.7$\pm$0.1} & \textbf{28.6$\pm$1.1} & \textbf{29.5$\pm$1.4} & \textbf{29.6$\pm$0.4} & \textbf{30.5$\pm$1.2} & \textbf{32.1$\pm$0.2} & \textbf{33.2$\pm$0.8} \\
\hline
 \multirow{5}{*}{\emph{S}$\rightarrow$\emph{M}}   & \multirow{5}{*}{70.2} & {FT} & 70.2$\pm$0.0 & 70.6$\pm$0.3 & 70.7$\pm$0.1 & 70.8$\pm$0.3 & 70.9$\pm$0.2 & 71.1$\pm$0.3 & 71.1$\pm$0.4 \\
 %& & {ATL} & 9.9$\pm$0.9 & 9.5$\pm$1.2 & 10.5$\pm$0.7 & 11.1$\pm$1.0 & 13.7$\pm$0.8 & 17.5$\pm$1.4 & 16.3$\pm$1.0 \\
 &  & {SHOT} & 72.6$\pm$1.9 & 73.6$\pm$2.0 & 74.1$\pm$0.6 & 74.6$\pm$1.2 & 74.9$\pm$0.7 & 75.4$\pm$0.3 & 76.1$\pm$1.5 \\
 &  & {S+F} & 74.4$\pm$1.5 & 83.1$\pm$0.7 & 83.3$\pm$1.1 & 85.9$\pm$0.5 & 86.0$\pm$1.2 & 87.6$\pm$2.6 & 89.1$\pm$1.0 \\
 &  & {T+F} & 74.2$\pm$1.8 & 81.6$\pm$4.0 & 83.4$\pm$0.8 & 82.0$\pm$2.3 & 86.2$\pm$0.7 & 87.2$\pm$0.8 & 88.2$\pm$0.6 \\
 &  & {TOHAN} & \textbf{76.0$\pm$1.9} & \textbf{83.3$\pm$0.3} & \textbf{84.2$\pm$0.4} & \textbf{86.5$\pm$1.1} & \textbf{87.1$\pm$1.3} & \textbf{88.0$\pm$0.5} & \textbf{89.7$\pm$0.5} \\
\hline
 \multirow{5}{*}{\emph{M}$\rightarrow$\emph{U}}   & \multirow{5}{*}{69.7} & {FT} & 74.4$\pm$0.7 & 76.7$\pm$1.9 & 76.9$\pm$2.2 & 77.3$\pm$1.1 & 77.6$\pm$1.4 & 78.3$\pm$2.1 & 78.3$\pm$1.6 \\
% & & {ATL} & 17.2$\pm$1.5 & 16.7$\pm$1.1 & 17.3$\pm$0.4 & 18.0$\pm$2.1 & 17.8$\pm$1.6 & 18.3$\pm$1.3 & 18.5$\pm$0.7 \\
 &  & {SHOT} & 87.2$\pm$0.2 & 87.9$\pm$0.3 & 87.8$\pm$0.4 & 88.0$\pm$0.4 & 87.9$\pm$0.5 & 88.0$\pm$0.3 & 88.4$\pm$0.3 \\
 &  & {S+F} & 83.7$\pm$0.9 & 86.0$\pm$0.4 & 86.1$\pm$1.1 & 86.5$\pm$0.8 & 86.8$\pm$1.4 & 87.0$\pm$0.6 & 87.2$\pm$0.8 \\
 &  & {T+F} & 84.2$\pm$0.1 & 84.2$\pm$0.3 & 85.2$\pm$0.9 & 85.2$\pm$0.6 & 86.0$\pm$1.5 & 86.8$\pm$1.5 & 87.2$\pm$0.5 \\
 &  & {TOHAN} & \textbf{87.7$\pm$0.7} & \textbf{88.3$\pm$0.5} & \textbf{88.5$\pm$1.2} & \textbf{89.3$\pm$0.9} & \textbf{89.4$\pm$0.8} & \textbf{90.0$\pm$1.0} & \textbf{90.4$\pm$1.2} \\
\hline
 \multirow{5}{*}{\emph{U}$\rightarrow$\emph{M}}   & \multirow{5}{*}{82.9} & {FT} & 83.5$\pm$0.4 & 84.3$\pm$2.4 & 84.5$\pm$0.7 & 85.5$\pm$1.3 & 86.6$\pm$1.0 & 87.2$\pm$0.7 & 88.1$\pm$2.7 \\
% & & {ATL} & 8.7$\pm$0.9 & 12.7$\pm$0.8 & 12.4$\pm$1.8 & 15.2$\pm$0.7 & 15.0$\pm$1.0 & 17.3$\pm$1.4 & 18.9$\pm$2.5 \\
 &  & {SHOT} & 83.1$\pm$0.5 & \textbf{85.5$\pm$0.3} & \textbf{85.8$\pm$0.6} & 86.0$\pm$0.2 & 86.6$\pm$0.2 & 86.7$\pm$0.2 & 87.0$\pm$0.1 \\
 &  & {S+F} & 83.2$\pm$0.2 & 84.0$\pm$0.3 & 85.0$\pm$1.2 & 85.6$\pm$0.5 & 85.7$\pm$0.6 & 86.2$\pm$0.6 & 87.2$\pm$1.1 \\
 &  & {T+F} & 82.9$\pm$0.7 & 83.9$\pm$0.2 & 84.7$\pm$0.8 & 85.4$\pm$0.6 & 85.6$\pm$0.7 & 86.3$\pm$0.9 & 86.6$\pm$0.7 \\
 &  & {TOHAN} & \textbf{84.0$\pm$0.5} & 85.2$\pm$0.3 & 85.6$\pm$0.7 & \textbf{86.5$\pm$0.5} & \textbf{87.3$\pm$0.6} & \textbf{88.2$\pm$0.7} & \textbf{89.2$\pm$0.5} \\
\hline
 \multirow{5}{*}{\emph{S}$\rightarrow$\emph{U}}   & \multirow{5}{*}{64.3} & {FT} & 64.9$\pm$1.1 & 66.5$\pm$1.5 & 66.7$\pm$1.7 & 67.3$\pm$1.1 & 68.1$\pm$2.3 & 68.3$\pm$0.5 & 69.7$\pm$1.4 \\
 %& & {ATL} & 11.7$\pm$1.8 &10.2$\pm$0.8  & 14.2$\pm$0.5 & 13.7$\pm$1.0 & 15.9$\pm$1.6 & 15.6$\pm$0.3 & 17.9$\pm$0.9 \\
 &  & {SHOT} & 74.7$\pm$0.3 & 75.5$\pm$1.4 & 75.6$\pm$1.0 & 75.8$\pm$0.7 & 77.1$\pm$2.1 & 77.8$\pm$1.6 & 79.6$\pm$0.6 \\
 &  & {S+F} & 72.2$\pm$1.4 & 73.6$\pm$1.4 & 74.7$\pm$1.4 & 76.2$\pm$1.3 & 77.2$\pm$1.7 & 77.8$\pm$3.0 & 79.7$\pm$1.9 \\
 &  & {T+F} & 71.7$\pm$0.6 & 74.3$\pm$1.9 & 74.5$\pm$0.8 & 75.9$\pm$2.1 & 77.7$\pm$1.5 & 76.8$\pm$1.8 & 79.7$\pm$1.9 \\
 &  & {TOHAN} & \textbf{75.8$\pm$0.9} & \textbf{76.8$\pm$1.2} & \textbf{79.4$\pm$0.9} & \textbf{80.2$\pm$0.6} & \textbf{80.5$\pm$1.4} &\textbf{81.1$\pm$1.1} & \textbf{82.6$\pm$1.9} \\
\hline
 \multirow{5}{*}{\emph{U}$\rightarrow$\emph{S}}   & \multirow{5}{*}{17.3} & {FT} & 23.4$\pm$1.8 & 23.6$\pm$2.7 & 23.8$\pm$1.6 & 24.6$\pm$1.4 & 24.6$\pm$1.2 & 24.8$\pm$0.7 & 25.5$\pm$1.8 \\
 %& & {ATL} & & & & & & & \\
 &  & {SHOT} & \textbf{30.3$\pm$1.2} & \textbf{31.6$\pm$0.4} & 29.8$\pm$0.5 & 29.4$\pm$0.3 & 29.7$\pm$0.5 & 29.8$\pm$0.8 & 30.1$\pm$0.9 \\
 &  & {S+F} & 28.1$\pm$1.2 & 28.7$\pm$1.3 & 29.0$\pm$1.2 & 30.1$\pm$1.1 & 30.3$\pm$1.3 & 30.7$\pm$1.0 & 30.9$\pm$1.5 \\
 &  & {T+F} & 27.5$\pm$1.4 & 27.9$\pm$0.9 & 28.4$\pm$1.3 & 29.4$\pm$1.8 & 29.5$\pm$0.7 & 30.2$\pm$1.0 & 30.4$\pm$1.7 \\
 &  & {TOHAN} & 29.9$\pm$1.2 & 30.5$\pm$1.2 & \textbf{31.4$\pm$1.1} & \textbf{32.8$\pm$0.9} & \textbf{33.1$\pm$1.0} & \textbf{34.0$\pm$1.0} & \textbf{35.1$\pm$1.8} \\
\bottomrule
\end{tabular}
\label{digits_results}
% \vspace{-1em}
\end{table*}

\section{Experiments}
We compare TOHAN with benchmark solutions on five standard supervised DA datasets: \emph{MNIST}$\left(\emph{M}\right)$, \emph{SYHN}$\left(\emph{S}\right)$, \emph{USPS}$\left(\emph{U}\right)$, \emph{CIFAR}-$10\left(\emph{CF}\right)$, \emph{STL}-$10\left(\emph{SL}\right)$. We follow the standard domain-adaptation protocols \citep{DBLP:conf/iclr/ShuBNE18} and compare average accuracy of $5$ independent repeated experiments. For digital datasets (i.e., \emph{M}, \emph{S}, and \emph{U}), we choose the number of target data (per class) from $1$ to $7$ \cite{motiian2017few}. For objects datasets (i.e., \emph{CF} and \emph{SL}), we choose the number of target data as $10$.
Details regarding these datasets can be found in Appendix~\ref{Asec:datasets}. The code is available at \href{https://github.com/Haoang97/TOHAN}{github.com/Haoang97/TOHAN}. 
%, \emph{CIFAR-10}$\left(\emph{C}\right)$ and \emph{STL-10}$\left(\emph{s}\right)$.

% \begin{table}[!t]
% 	\centering
% 	\small
% 	\setlength{\tabcolsep}{1.5mm}
	
% 	% \resizebox{0.92\columnwidth}{!}{%
% 	\begin{tabular}{  l  c  c  c  c  c  c  c  c   }
% 		\toprule
% 		 FHA & \multicolumn{7}{c}{{Number of Target Data per Class}} &\multicolumn{1}{c}{\multirow{2}{*}{Average}}\\
% 		\cline{2-8}
% 		 Methods & {1} & {2}  & {3} & {4} & {5}  & {6} & {7} \\
% 		\midrule
% 		%\multirow{4}{*}
% 		 {S+F} & 61.2 & 63.0 & 64.3 & 65.4 & 65.7 & 66.4 & 67.2 & 64.7\\
% 		  {T+F}& 61.0 & 63.0 & 64.2 & 64.5 & 65.7 & 66.5 & 67.4 & 64.6\\
% 		  {ST+F} & 61.8 & 64.5 & 64.9 & 65.8 & 66.5 & 67.3 & 68.4 & 65.5\\
% 		  {TOHAN} & \textbf{63.3} & \textbf{65.4} & \textbf{66.4} & \textbf{67.5} & \textbf{68.0} & \textbf{68.9} & \textbf{70.0} & \textbf{67.1} \\
% 		 \bottomrule
% 	\end{tabular}
% 	\caption{Ablation study. We show the average accuracy of $6$ tasks on digits datasets in this table. Bold value represents the highest accuracy ($\%$) on each column. Since each element in this table is the average accuracy of a method on $6$ tasks, we don't present the standard deviation of the accuracy in this table. The full results regarding the ablation study can be found in Appendix~\ref{Asec:Details}.}\label{ablation_study}
% 	\vspace{-1em}
% \end{table}

\paragraph{Benchmark solutions for FHA.}
Although the FHA is a new problem setting, we still design $5$ benchmark solutions to this new problem. (1) \textit{Without adaptation} (WA): to classify the target domain with the source classifier (encoder \emph{$g_s$} and classifier \emph{$h_s$}). (2) \textit{Fine-tuning} (FT): to train the \emph{classifier} \emph{$g_s$} with few owned target data. (3) \textit{SHOT}: a novel HTL method, where we modify it to use the labeled target data instead of only using the unlabeled target data. \citep{DBLP:conf/icml/LiangHF20}. (4) \textit{S+FADA} (S+F): to generate faked source data with the source classifier then apply them to DANN \citep{ganin2016domain}. (5) \textit{T+FADA} (T+F): to generate fake target data with few real target data then apply them to DANN. We demonstrate details of $5$ benchmark solutions in Appendix~\ref{Asec:BenSolution}. Experimental details can be found in Appendix~\ref{Asec:Details}. Moreover, we conduct additional experiments to compare existing HTL method named dkdHTL \citep{yu2020dynamic}, and the related results and analysis can be found in Appendix~\ref{Asec:htl}.

\begin{wrapfigure}{r}{0.6\textwidth}
	\centering
	\vspace{-18pt}
\subfigure[\tiny{$CF$ $\to$ $SL$}]{
\includegraphics[width=0.27\textwidth]{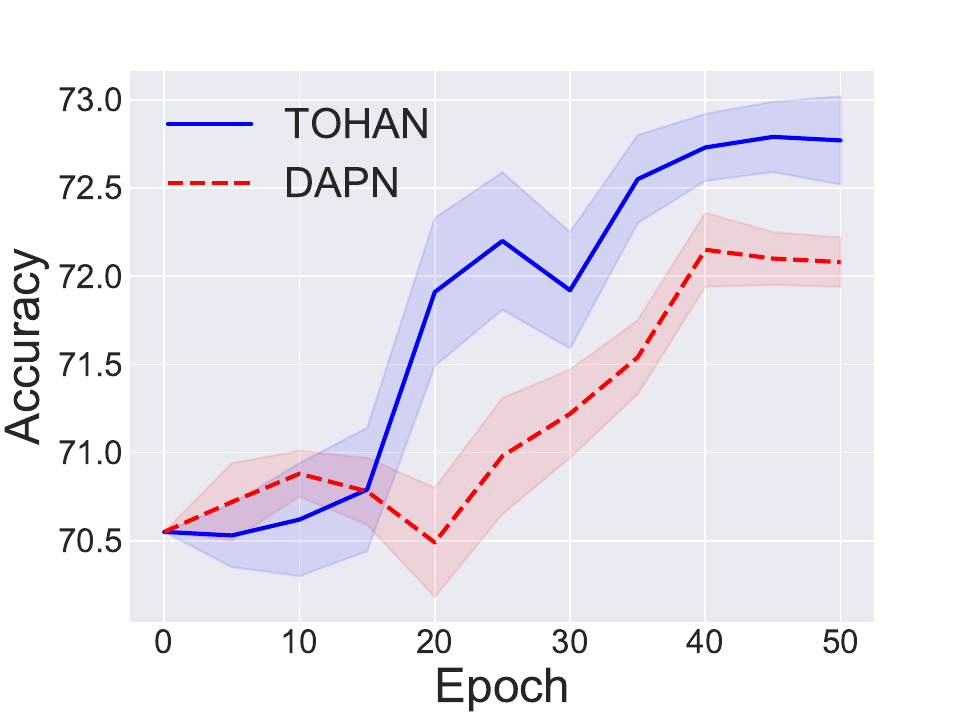}
}
\subfigure[\tiny{$SL$ $\to$ $CF$}]{
\includegraphics[width=0.27\textwidth]{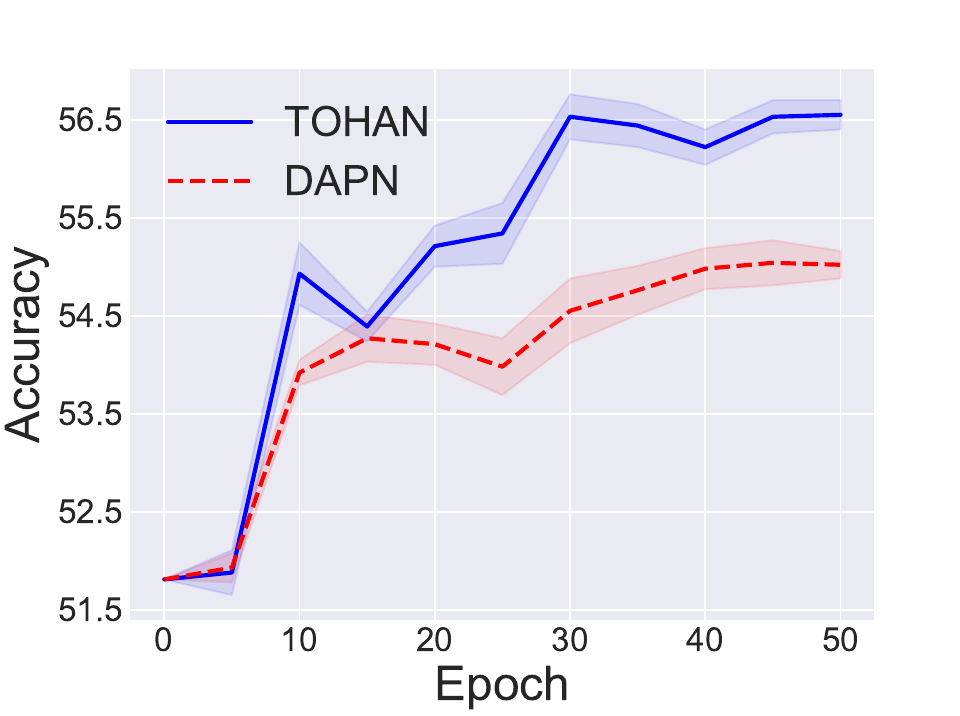}
}
\vspace{-10pt}
%\captionsetup{font={scriptsize}}
    \caption{TOHAN vs DAPN.}
\vspace{-15pt}
	\label{fig:few-shot}
\end{wrapfigure}


% \subsection{Results analysis on FHA tasks}
\paragraph{Results on digits FHA tasks.} We conduct experiments on $6$ digits FHA tasks: \emph{M}$\rightarrow$\emph{S}, \emph{S}$\rightarrow$\emph{M}, \emph{M}$\rightarrow$\emph{U}, \emph{U}$\rightarrow$\emph{M}, \emph{S}$\rightarrow$\emph{U} and \emph{U}$\rightarrow$\emph{S}. Table~\ref{digits_results} reports the target-domain classification accuracy of $6$ methods on $6$ digits FHA tasks. It is clear that TOHAN performs the best on almost every task. On \emph{M}$\rightarrow$\emph{S}, \emph{S}$\rightarrow$\emph{M}, \emph{M}$\rightarrow$\emph{U} and \emph{S}$\rightarrow$\emph{U}, TOHAN outperforms all benchmark solutions obviously. However, on the tasks \emph{U}$\rightarrow$\emph{M} and \emph{U}$\rightarrow$\emph{S}, the accuracy of TOHAN is slightly lower than SHOT when the amount of target data is too small ($n=1,2$). This abnormal phenomenon shows that TOHAN cannot generate intermediate domain data effectively with very little target data, especially when the resolution of source data is much smaller than that of target data. In this case, the data we generate is close to the source domain, so TOHAN degrades to S+FADA. 

In Appendix~\ref{Asec:AddAna}, we use t-SNE to visualize the features extracted by TOHAN and $5$ benchmark solutions on \emph{M}$\rightarrow$\emph{U} task (see Figure~\ref{fig:t-sne} in Appendix~\ref{Asec:AddAna}). When we use WA and FT methods, nearly all classes mix together. Although the classification accuracies of SHOT, S+F and T+F are relatively high, there are still some mixing among classes. For TOHAN, it can be seen that all classes are separated well, which demonstrates that TOHAN works well for solving the FHA problem.

\paragraph{Results on objects FHA tasks.} Following \cite{DBLP:conf/iclr/ShuBNE18}, we also evaluate TOHAN and benchmark solutions on $2$ objects FHA tasks: $SL$ $\rightarrow$ $CF$ and $CF$ $\rightarrow$ $SL$, and the results are shown in Table~\ref{tab:object}. Considering the complexity of datasets and the difficulty of our problem setting, we do not have amazing results like digits tasks. In $SL$ $\rightarrow$ $CF$, we achieve of $4.8\%$ improvement over WA and a performance accuracy of $56.9\%$. Note that because the numbers of pixels per image of $CF$ and $SL$ are quite different, the images from $SL$ lose a lot of information when inputted to the pre-trained model of $CF$, thus making the effects of TOHAN and benchmark solutions are not obvious for $CF$ $\rightarrow$ $SL$.
%%%%%%%%%%%%%%%%%%%%%%%%%%%%%%%%%%%%%%%%%%%%%%%
\iffalse
\begin{table}[!t]
	\centering
	\footnotesize
% 	\small
	\setlength{\tabcolsep}{1.5mm}
	\caption{Ablation study. We show the average accuracy of $6$ tasks on digits datasets in this table. Bold value represents the highest accuracy ($\%$) on each column. Since each element in this table is the average accuracy of a method on $6$ tasks, we don't present the standard deviation of the accuracy in this table. The full results regarding the ablation study can be found in Appendix~\ref{Asec:AddAna}.}\label{ablation_study}
	% \resizebox{0.92\columnwidth}{!}{%
	\begin{tabular}{  l  c  c  c  c  c  c  c  c   }
		\toprule
		 FHA & \multicolumn{7}{c}{{Number of Target Data per Class}} &\multicolumn{1}{c}{\multirow{2}{*}{Average}}\\
		\cline{2-8}
		 Methods & {1} & {2}  & {3} & {4} & {5}  & {6} & {7} \\
		\midrule
		%\multirow{4}{*}
		 {S+F} & 61.2 & 63.0 & 64.3 & 65.4 & 65.7 & 66.4 & 67.2 & 64.7\\
		  {T+F}& 61.0 & 63.0 & 64.2 & 64.5 & 65.7 & 66.5 & 67.4 & 64.6\\
		  {ST+F} & 61.8 & 64.5 & 64.9 & 65.8 & 66.5 & 67.3 & 68.4 & 65.5\\
		  {TOHAN} & \textbf{63.3} & \textbf{65.4} & \textbf{66.4} & \textbf{67.5} & \textbf{68.0} & \textbf{68.9} & \textbf{70.0} & \textbf{67.1} \\
		 \bottomrule
	\end{tabular}
	
	\vspace{-1em}
\end{table}
\fi
%%%%%%%%%%%%%%%%%%%%%%%%%%%%%%%%%%%%%%%%%%%%%%%%%%%

\paragraph{Comparing TOHAN with FSL methods.}
As mentioned above, FHA is a difficult case of FSL where the prior knowledge is a pre-trained model of another domain. To test the effectiveness of FSL methods in FHA, we compare TOHAN with a novel FSL method called \emph{domain-adaptive few-shot learning} (DAPN) \citep{zhao2020domain}. Note that we use the same pre-trained model in both TOHAN and DAPN. Taking $CF$ $\leftrightarrow$ $SL$ with five target data (per class) as an example, we solve FHA with TOHAN and DAPN and show the results in Figure~\ref{fig:few-shot}. It is clear that TOHAN outperforms DAPN when the training epoch ($t$) is relatively large.

\begin{wrapfigure}{r}{0.29\textwidth}
\vspace{-18pt}
  \begin{center}
   \includegraphics[width=0.25\textwidth]{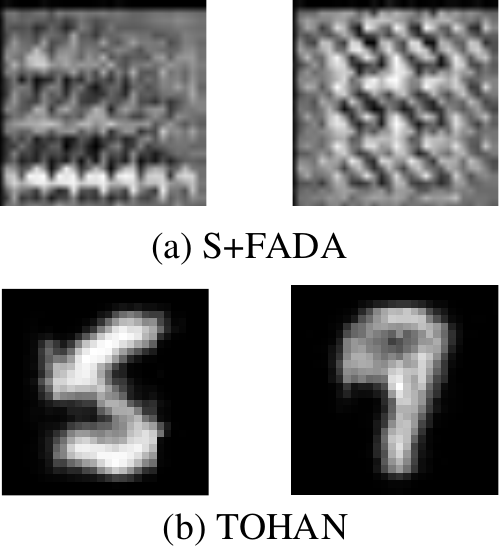}
  \end{center}
  \vspace{-11pt}
  \caption{Visualization of S+FADA and TOHAN.}
  \vspace{-20pt}
  \label{fig:visual}
\end{wrapfigure}

\paragraph{Ablation Study.}
Finally, we study the advantages of one-step method over other two-step methods. We consider the following baselines: S+F, T+F and \emph{ST+FADA} (ST+F). We have explained S+F and T+F previously. ST+F denotes the two-step version of TOHAN, i.e., to conduct intermediate  domain  generation and intermediate-to-target  distributional  adaptation separately. We make ablation study on three digital datasets mentioned before as an example.

As shown in Table~\ref{ablation_study}, it is clear that TOHAN works better than the other baselines. The generator of S+F uses the loss $\mathcal{L}_{G_n}^{s}$, which merely contains knowledge from the source domain. The generator of T+F uses the loss $\mathcal{L}_{G_n}^{t}$ and ignores the knowledge contained in the source-domain classifier. In contrast, TOHAN uses both $\mathcal{L}_{G_n}^{s}$ and $\mathcal{L}_{G_n}^{t}$. As a result, TOHAN achieves higher accuracy than S+F and T+F. Besides, the generators and classifiers in TOHAN promote each other in the training process, which results in that TOHAN performs better than the ST+F. In Figure~\ref{fig:visual}, we visualize the data generated by S+FADA and TOHAN. It is clear that data generated by S+FADA are chaotic that contain little useful information. However, data generated by TOHAN contain many target-domain high-level visual features, and they can be classified by the source classifier accurately, resulting in a better performance in FHA. The detailed analysis of ablation study can be found in Appendix~\ref{Asec:AddAna}.

\begin{table}[!t]
  \centering
  \small
  \setlength{\belowcaptionskip}{1mm}
  \caption{Classification accuracy$\pm$standard deviation ($\%$) on 2 objects FHA tasks: CIFAR-$10\to$ STL-$10$ (\textit{CF}$\to$\textit{SL}) and STL-$10\to$ CIFAR-$10$ (\textit{SL}$\to$\textit{CF}). Bold value represents the highest accuracy ($\%$) among TOHAN and benchmark solutions.}
  \vspace{1mm}
  %\small
    \begin{tabular}{ccccccc}
    \toprule
       Methods   & WA    & FT & SHOT  & S+F   & T+F   & TOHAN \\
    \midrule
    \textit{CF}$\to$\textit{SL} & 70.6 & 71.5$\pm$1.0 & 71.9$\pm$0.4 & 72.1$\pm$0.4 & 71.3$\pm$0.5 & \textbf{72.8$\pm$0.1} \\
    \midrule
    \textit{SL}$\to$\textit{CF} & 51.8 & 54.3$\pm$0.5 & 53.9$\pm$0.2 & \textbf{56.9$\pm$0.5} & 55.8$\pm$0.8 & 56.6$\pm$0.3 \\
    \bottomrule
    \end{tabular}%
    
  \label{tab:object}%
   \vspace{-1em}
\end{table}%

\begin{table}[!t]
	\centering
	%\footnotesize
   	\small
	\setlength{\tabcolsep}{3.5mm}
	\caption{Ablation study. We show the average accuracy of the $6$ tasks on digits datasets in this table. Bold value represents the highest accuracy ($\%$) on each column. See full results in Appendix~\ref{Asec:AddAna}.}\label{ablation_study}
	\vspace{1mm}
	% \resizebox{0.92\columnwidth}{!}{%
	\begin{tabular}{  l  c  c  c  c  c  c  c }
		\toprule
		 FHA & \multicolumn{7}{c}{{Number of Target Data per Class}} \\
		\cline{2-8}
		 Methods & {1} & {2}  & {3} & {4} & {5}  & {6} & {7} \\
		\midrule
		%\multirow{4}{*}
		 {S+F} & 61.2 & 63.0 & 64.3 & 65.4 & 65.7 & 66.4 & 67.2 \\
		  {T+F}& 61.0 & 63.0 & 64.2 & 64.5 & 65.7 & 66.5 & 67.4 \\
		  {ST+F} & 61.8 & 64.5 & 64.9 & 65.8 & 66.5 & 67.3 & 68.4 \\
		  {TOHAN} & \textbf{63.3} & \textbf{65.4} & \textbf{66.4} & \textbf{67.5} & \textbf{68.0} & \textbf{68.9} & \textbf{70.0}  \\
		 \bottomrule
	\end{tabular}
	
% 	\vspace{-1em}
\end{table}

\paragraph{Verification of No Source-data Leakage in Intermediate Domain.}
As a key contribution, TOHAN solves FHA through generating intermediate data. To guarantee that no source data are leaked, we need to verify that there is no source-domain features in the intermediate data. We determine this by calculating the PSNR values \cite{DBLP:conf/icpr/HoreZ10,DBLP:conf/cvpr/YinMVAKM21} between each intermediate sample and all source samples. PSNR indicates the generation quality of an image $f$ given a standard image $g$, and is defined as
\begin{align*}
    \label{eq:psnr}
    \textnormal{PSNR}(f,g)=10\log_{10}\frac{255^2}{\textnormal{MSE}(f,g)},\ {\rm where}\ \textnormal{MSE}(f,g)=\frac{1}{MN}\sum_{i=1}^{M}\sum_{j=1}^{N}(f_{ij}-g_{ij})^2.
\end{align*}

\begin{wrapfigure}{r}{0.5\textwidth}
    \vspace{-1em}
    \centering
    \includegraphics[width=0.5\textwidth]{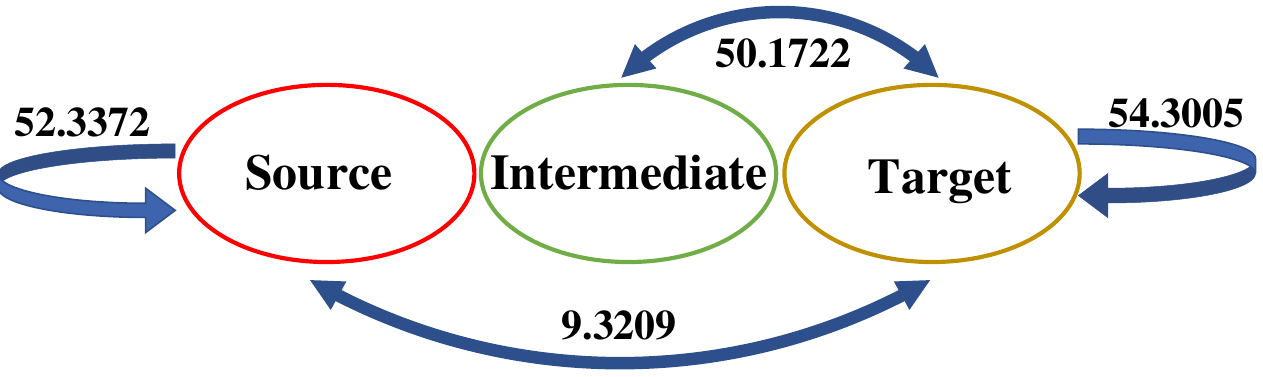}
	\caption{The average PSNR values of (source data, source data), (target data, target data), (source data, target data), and (intermediate data, target data).}
	\label{fig: psnr}
 	\vspace{-1em}
\end{wrapfigure}

The larger PSNR value indicates the two images are more similar. Therefore, taking \emph{M}$\to$\emph{S} as an example, we report the top-$5$ largest PSNR values in Table~\ref{tab:5-psnr}. That is, we check whether the worst case satisfies our claim. For comparison, we also compute the PSNR values of (source data, source data), (target data, target data), (source data, target data), and (intermediate data, target data), and we report the \emph{average} PSNR values of the above four cases in Figure~\ref{fig: psnr}. As can be seen, the intermediate data are much closer to the target data, and they are very different from the source data. The average PSNR between source data and target data is $9.3209$. The top-$5$ largest PSNR values between each intermediate data and all source data ($\approx17.89$) are obviously smaller than $50.1722$ (the average PSNR between intermediate data and target data). Through this result, we can state that intermediate data are similar to the target data and very different from the source data. Therefore, the above evidence shows that the generated intermediate data contain no source domain features, and the source data do not leak when generating the intermediate data.

\begin{table}[!t]
	\centering
	%\footnotesize
   	\small
	\caption{The top-5 largest PSNR values between intermediate samples and all source samples.}
	\vspace{1mm}
	\label{tab:5-psnr}
	% \resizebox{0.92\columnwidth}{!}{%
	\begin{tabular}{c  c  c  c  c  c}
	    \hline
	    Ranking & $1$ & $2$ & $3$ & $4$ & $5$ \\
		\hline
		PSNR value & 17.8951 & 17.8948 & 17.8948 & 17.8947 & 17.8947 \\
		%& 17.8944 & 17.8941 & 17.8940 & 17.8940 & 17.8939\\
		\hline
	\end{tabular}
 	\vspace{-1em}
\end{table}

\section{Conclusion}
This paper presents a very challenging problem setting called \emph{few-shot hypothesis adaptation} (FHA), which trains a target-domain classifier with only few labeled target data and a well-trained source-domain classifier. Since we can only access a well-trained source-domain classifier in FHA, the private information in the source domain are protected well. To this end, we propose a novel one-step FHA method, called \emph{target-oriented hypothesis adaptation network} (TOHAN). Experiments conducted on $8$ FHA tasks confirm that TOHAN effectively adapts the source-domain classifier to the target domain and outperforms competitive benchmark solutions to the FHA problem. 
% \section{ Acknowledge }
\begin{ack}
% \paragraph{Acknowledgments.}
This work was partially supported by the National Natural Science Foundation of China (No.  91948303-1, No. 61803375, No. 12002380, No. 62106278, No. 62101575, No. 61906210), the National Grand R\&D Plan (Grant No. 2020AAA0103501), and the National Key R$\&$D Program of China (2021ZD0140301). FL would also like to thank Dr. Yanbin Liu for productive discussions.
\end{ack}

\bibliography{example_paper}
\bibliographystyle{plain}

\clearpage

\appendix
\onecolumn

\section{Related Work}
In this section, we briefly review \emph{few-shot learning} (FSL) and two  domain adaptation settings related to the FHA problem, which include FDA, and \emph{source-data-free UDA} (SFUDA).

% \begin{figure*}[!t]
% 	\centering
% 	\subfigure[Fine-tuning]{
% 		\begin{minipage}[b]{0.24\textwidth}
% 			\includegraphics[width=0.7\textwidth]{baseline1.png} 
% 		\end{minipage}
% 		\label{fig:baseline1}
% 	}
% 	\subfigure[SHOT]{
%     		\begin{minipage}[b]{0.24\textwidth}
%   		 	\includegraphics[width=0.7\textwidth]{baseline2.png}
%     		\end{minipage}
% 		\label{fig:baseline2}
%     }
    
% % \vfill
% 	\subfigure[Source hypothesis  adaptation  network]{
% 		\begin{minipage}[b]{0.24\textwidth}
% 		\includegraphics[width=0.9\textwidth]{baseline3.png} 
% 		\end{minipage}
% 		\label{fig:baseline3}
% 	}
%     	\subfigure[Target hypothesis  adaptation  network]{
%     		\begin{minipage}[b]{0.24\textwidth}
% 		 	\includegraphics[width=0.9\textwidth]{baseline4.png}
%     		\end{minipage}
% 		\label{fig:baseline4}
%     	}
% 	\caption{Overview of our benchmark solutions.}
% 	\label{fig:baseline}
% \end{figure*}

\textbf{Few-shot Learning.}
% Recently, few-shot learning (FSL) \citep{DBLP:conf/ijcai/LiWHSGL20} has become a popular topic. With the success of deep learning-based approaches in the data-abundant setting, we are curious about how to generalize such deep learning approaches to the few-shot setting.
Existing FSL methods can be divided into three categories: (1) Augmenting training data set by prior knowledge. Data augmentation via hand-crafted rules serves as pre-processing in FSL methods. For instance, we can use reflection \cite{DBLP:conf/iclr/EdwardsS17}; and (2) Constraining hypothesis space by prior knowledge  \cite{DBLP:conf/iclr/MishraR0A18}; and (3) Altering search strategy in hypothesis space by prior knowledge. For instance, we can use early-stopping \cite{DBLP:conf/nips/ArikCPPZ18}. Note that our method belongs to category (1). However, the prior knowledge we have is more difficult to leverage than the prior knowledge that FSL methods have.

% \subsection{Supervised Domain Adaptation} 
% SDA methods can deal with the covariate shift existed in two domains.  \citet{motiian2017unified} propose a unified framework for addressing the problem of visual supervised domain adaptation and generalization with deep models. 
% % This approach has a high speed of adaptation and can be extended to domain generalization.
% \citet{luo2017label} propose a framework that learns transferable representations across different domains and tasks in a label efficient manner. It uses an end-to-end SDA method to tackle the problem of high sensitivity and overfitting during fine-tuning stage. Though SDA can achieve excellent performance, abundant labeled target data are not easy to obtain in the real-world scenario. 
% % If we have few labeled target samples, part of technique in SDA is still available.
\textbf{Few-shot Domain Adaptation.}
With the development of FSL, researchers also apply ideas of FSL into domain adaptation, called \emph{few-shot domain adaptation} (FDA). 
% With limited labeled target data, there are no guarantee of the performance of target-domain classifier. The existing works about FDA has some solution to this problem. 
FADA \citep{motiian2017few} is a representative FDA method, which pairs data from source domain and data from target domain and then follows the adversarial domain adaptation method. Casual mechanism transfer \citep{DBLP:conf/icml/TeshimaSS20} is another novel FDA method dealing with a meta-distributional scenario, in which the data generating mechanism is invariant among domains. Nevertheless, FDA methods still need to access many labeled source data for training, which may cause the private-information leakage of the source domain. 
% We would like to see whether we can effectively finish FDA without source data.

\textbf{Hypothesis Transfer Learning.}
In the \emph{hypothesis transfer learning} (HTL), we can only access a well-trained source-domain classifier and small labeled or abundant unlabeled target data. \cite{DBLP:conf/icml/KuzborskijO13} requires small labeled target data and uses the Leave-One-Out error find the optimal transfer parameters. Later, SHOT \cite{DBLP:conf/icml/LiangHF20} is proposed to solve the HTL with many unlabeled target data by freezing the source-domain classifier and learning a target-specific feature extraction module. \cite{hou2020source} proposes an image translation method that transfers the style of target images to that of unseen source images. 
As for the universal setting, a two-stage learning process \citep{kundu2020universal} has been proposed to address the HTL problem. Compared with FHA, HTL still requires at least small target data (e.g., at least $12$ samples in binary classification problem \cite{DBLP:conf/icml/KuzborskijO13}, or at least two of labeling percentage \cite{DBLP:conf/cvpr/AhmedLPR20}). In FHA, we focus on a more challenging situation: only few data (e.g., one sample per class) are available.
% \subsection{Hypothesis Transfer}

\section{Proof of Theorem~\ref{thm:1}}
\label{Asec:Thm}
We state here two known generalization bounds \cite{book} used in our proof.

\begin{lemma}
\label{lemma:1}
Suppose that $\mathcal{H}$ is a set of functions from $\mathcal{X}$ to $\{0,1\}$ with finite $VC$-dimension $V\ge1$. For any distribution $P$ over $\mathcal{X}$, any target function, and any $\epsilon$, $\delta>0$, if we draw a set of data from $P$ of size
\begin{equation}\nonumber
    m(\epsilon,\delta,V)=\frac{64}{\epsilon^2}\left(2V\ln\left(\frac{12}{\epsilon}\right)+\ln\left(\frac{4}{\delta}\right)\right),
\end{equation}
then with probability at least $1-\delta$, we have $|err(h)-\widehat{err}(h)|\le\epsilon$ for all $h\in\mathcal{H}$.
\end{lemma}

\begin{lemma}
\label{lemma:2}
Suppose that $\mathcal{H}$ is a set of functions from $\mathcal{X}$ to $\{0,1\}$ with finite $VC$-dimension $V\ge1$. For any probability distribution $P$ over $\mathcal{X}$, any target function $c^*$, we have
\begin{equation}\nonumber
    \bm{Pr}\left[\sup\limits_{h\in\mathcal{H},\widehat{err}(h)=0}|err(h)-\widehat{err}(h)\ge\epsilon|\right]\le2\mathcal{H}[2m,P]e^{-m\epsilon/2}.
\end{equation}
So, for any $\epsilon$, $\delta>0$, if we draw a set of data from $P$ of size
\begin{equation}\nonumber
    m\ge\frac{2}{\epsilon}\left(2\ln(\mathcal{H}[2m,P])+\ln\left(\frac{2}{\delta}\right)\right),
\end{equation}
then with probability at least $1-\delta$, we have that all functions with $\widehat{err}(h)=0$ satisfy
\begin{equation}\nonumber
    err(h)\le\epsilon.
\end{equation}
\end{lemma}

Now we begin the proof of Theorem~\ref{thm:1}.
\begin{proof}
Let $S$ be the set of $m_u$ unlabeled data. By standard VC-dimension bounds (e.g., Lemma~\ref{lemma:1}), the number of unlabeled data given is sufficient to ensure that with probability at least $1-\frac{\delta}{2}$ we have
\begin{equation}\nonumber
    |\bm{Pr}_{x\sim\widebar{S}}[\chi_h(x)=1]-\bm{Pr}_{x\sim P}[\chi_h(x)=1]|\le\epsilon\quad \text{for all }\chi_h\in\chi(\mathcal{H}),
\end{equation}
where $\widebar{S}$ denotes the uniform distribution over $S$.

Since $\chi_h(x)=\chi(h,x)$, this implies that we have
\begin{equation}\nonumber
    |\chi(h,D)-\hat{\chi}(h,S)|\le\epsilon\quad \text{for all }h\in\mathcal{H}.
\end{equation}
Therefore, the set of hypotheses with $\hat{\chi}(h,S)\ge1-t-\epsilon$ is contained in $\mathcal{H}_{P,\chi}(t+2\epsilon)$.

The bound on the number of labeled data now follows directly from known concentration results using the expected number of partitions instead of the maximum in the standard VC-dimension bounds (e.g., Lemma~\ref{lemma:2}). This bound ensures that with probability $1-\frac{\delta}{2}$, none of the functions $h\in\mathcal{H}_{P,\chi}(t+2\epsilon)$ with $err(h)\ge\epsilon$ have $\widehat{err}(h)=0$.

The above two arguments together imply that with probability $1-\delta$, all $h\in\mathcal{H}$ with $\widehat{err}(h)=0$ and $\hat{\chi}(h,S)\ge1-t-\epsilon$ have $err(h)\ge\epsilon$, and furthermore $c^*$ has $\hat{\chi}(c^*,S)\ge1-t-\epsilon$. This in turn implies that with probability at least $1-\delta$, we have $err(\hat{h})\le\epsilon$, where
\begin{equation}\nonumber
    \hat{h}=\mathop{\arg\max}_{h\in\mathcal{H}_0}\hat{\chi}(h,S).
\end{equation}

\end{proof}

\section{Datasets}
\label{Asec:datasets}
\paragraph{Digits.} Following the evaluation protocol of \cite{motiian2017few}, we conduct experiments on $6$ adaptation scenarios: \emph{M}$\rightarrow$\emph{S}, \emph{S}$\rightarrow$\emph{M}, \emph{M}$\rightarrow$\emph{U}, \emph{U}$\rightarrow$\emph{M}, \emph{S}$\rightarrow$\emph{U} and \emph{U}$\rightarrow$\emph{S}. MNIST \citep{lecun1998gradient} images have been size-normalized and centered in a fixed-size ($28\times 28$) image. USPS \citep{hull1994database} images are $16\times 16$ grayscale pixels. SVHN \citep{netzer2011reading} images are $32\times 32$ pixels with $3$ channels.

\paragraph{Objects.} We also evaluate TOHAN and benchmark solutions on CIFAR-$10$ \citep{krizhevsky2009learning} and STL-10 \citep{coates2011stl10}, following \cite{DBLP:conf/iclr/ShuBNE18}. The CIFAR-10 dataset contains $60,000$ $32\times32$ color images in 10 categories. The STL-10 dataset is inspired by the CIFAR-10 dataset but with some modifications. However, these two datasets only contain nine overlapping classes. We removed the non-overlapping classes (“frog” and “monkey”) \citep{DBLP:conf/iclr/ShuBNE18}.

\begin{figure*}[!t]
	\centering
	\subfigure[Fine-tuning]{
		\includegraphics[width=0.21\textwidth]{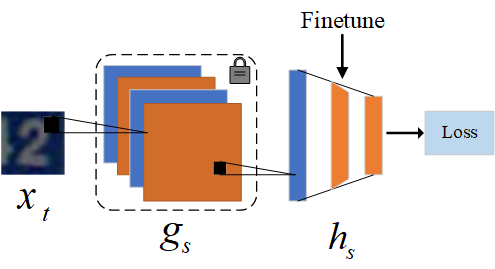} 
		\label{fig:baseline1}
	}
	\subfigure[SHOT]{
   	\includegraphics[width=0.21\textwidth]{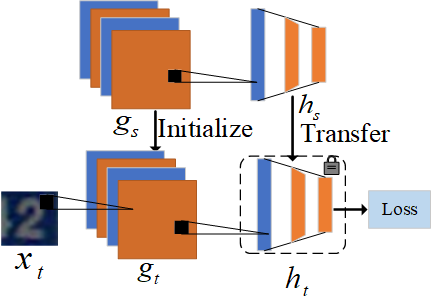}
		\label{fig:baseline2}
    }
% \vfill
	\subfigure[S+FADA]{

		\includegraphics[width=0.23\textwidth]{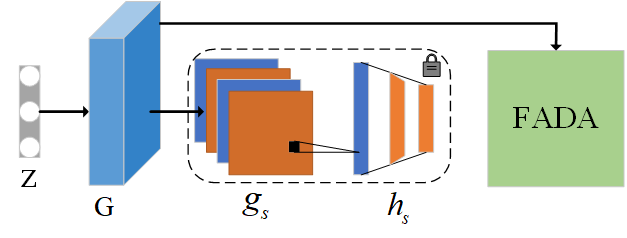} 
		\label{fig:baseline3}
	}
    	\subfigure[T+FADA]{
    	\includegraphics[width=0.21\textwidth]{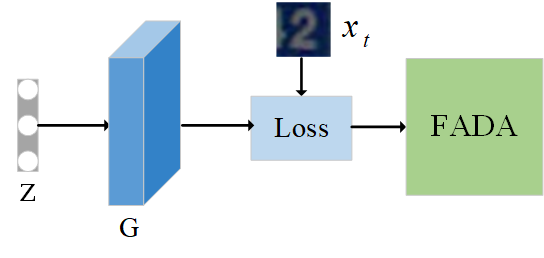}
		\label{fig:baseline4}
		%\vspace{-1em}
    	}
	\caption{Overview of benchmark solutions to the FHA problem. (a) We freeze the source encoder \emph{$g_s$} and train the source classifier \emph{$h_s$} with the target data $D_t$. (b) We first train source encoder \emph{$g_s$} and classifier \emph{$h_s$} , then we transfer them to the target domain. We generate pseudo labels for target data, then we use them to train the target model with classifier \emph{$h_t$} freezed. (c) We generate some source-like data under the guidance of the source classifier, then we combine them with FADA. (d) We generate some data close to target domain, i.e. decreasing the distance between data and target domain, then we combine them with FADA.}
	\label{fig:baseline}
	\vspace{-1em}
\end{figure*}

\section{Benchmark Solutions for FHA}
\label{Asec:BenSolution}
To solve the FHA problem, this section presents $5$ benchmark solutions that directly combine existing techniques used in the deep learning and domain adaptation fields.

\paragraph{Without adaptation.} Since we have a source-domain classifier, we can directly use it to classify the target data, which is a frustrating solution to the FHA problem.

\paragraph{Fine-tuning.} See Figure~\ref{fig:baseline1}. Fine-tuning is a basic solution to the FHA problem. We freeze the source encoder \emph{$g_s$} and train the source classifier \emph{$h_s$} with the target data $D_t$. In this way, knowledge about target domain is filled into source hypothesis. 

%\paragraph{ATL.} \emph{Analogical transfer learning} (ATL) is a few-shot domain adaptation method, which follows a analogical strategy. ATL first learns a revissd source hypothesis with only instances contributing to the target hypothesis. Then, ALT transfers both the revised source hypothesis and the target hypothesis to learn an analogical hypothesis. Note that the classifiers of ATL are SVMs.

\paragraph{SHOT.} See Figure~\ref{fig:baseline2}. SHOT is a novel method for source hypothesis transfer \citep{DBLP:conf/icml/LiangHF20}. It learns the optimal target-specific feature learning module to fit the source hypothesis with only the source classifier. We first train source encoder \emph{$g_s$} and classifier \emph{$h_s$} , and then we transfer them to the target domain. SHOT is an UDA method. Thus, we generate pseudo labels for target data, and then we use them to train the target model with classifier \emph{$h_t$} freezed. Although SHOT is suitable for our FHA problem, it requires a lot of target data, which is an obstacle for FHA.
%They firstly generate source classifier with source data, then they abandon source data and transfer the model to the target domain. The parameters of target encoder \emph{$g_t$} is initialized by source encoder \emph{$g_s$}, and target classifier \emph{$h_t$} is a full copy of source classifier \emph{$h_s$} with parameters freezed during training. Information maximization and self-supervised pseudo-labeling are two tricks applied in SHOT to improve performance.

\paragraph{S+FADA.} See Figure~\ref{fig:baseline3}. As mentioned in Figure~\ref{fig:intro}, a straightforward solution to the FHA problem is a two-step approach. We can train a source-data generator \emph{G} under the guidance of source hypothesis, and then we use it to generate source data. First, we input Gaussian random noise \emph{z} to \emph{G}, then \emph{G} outputs various disordered data. Second, these data is inputted into \emph{$g_s$}$\circ$\emph{$h_s$}, and then \emph{$h_s$} outputs the probability of $G(z)$ belonging to each class. Third, if we would like to generate data belonging to $n^{th}$ class, we should optimize $G$ to push the probability of $G(z)$ belonging to $n$ near to $1$. Finally, we can apply the restored source data into an adversarial DA method to train a target domain classifier \emph{$h_t$}. 

\begin{figure}[!tp]
     \centering
     \includegraphics[width=0.7\textwidth]{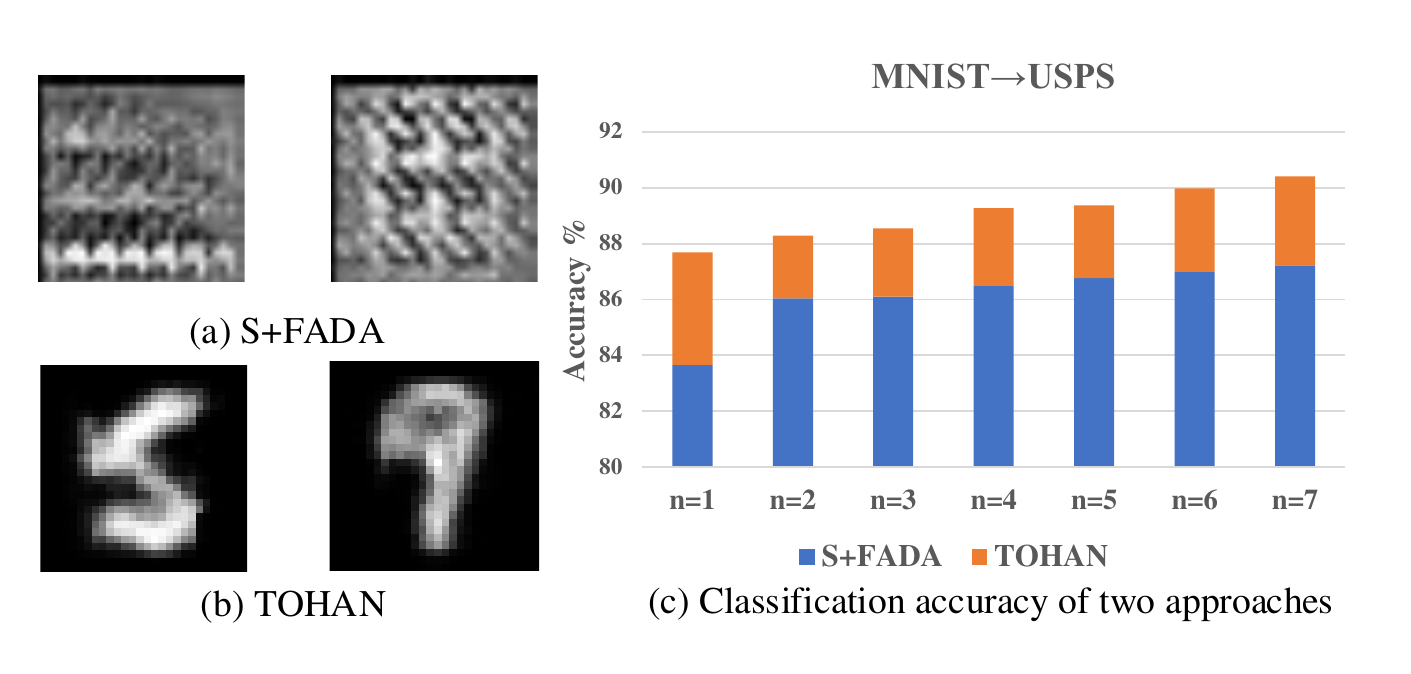}
     \vspace{-0.5em}
 	\caption{A straightforward solution for FHA is a two-step approach. Namely, we can first generate source data and then train a target-domain classifier using the generated source data and an FDA method (e.g., \emph{few-shot  adversarial  domain  adaptation} (FADA)). 
 	A visualized comparison between S+FADA and TOHAN (take MNIST$\rightarrow$USPS as an example) is displayed in subfigures (a) and (b). On the left side, 
 	Subfigure (a) illustrates source-domain data generated by a two-step method: S+FADA. It is clear that the generated data are just noise and do not contain useful information about the source domain. In subfigure (b), we illustrate the intermediate-domain data generated by our method (i.e., TOHAN). It is clear that the generated intermediate-domain data contain useful information about two domains. In subfigure (c), the histogram shows classification accuracy of the two methods, and TOHAN outperforms S+FADA clearly.}
 	\label{fig:intro}
 	\vspace{-1em}
 \end{figure}

\paragraph{T+FADA.} See Figure~\ref{fig:baseline4}. Different from S+FADA, we train a generator with the help of target data instead, and then we generate data close to target domain. We input Gaussian random noise \emph{z} to generator \emph{G} and minimize the distance between $G(z)$ and target data. Finally, we sequentially combine these generated data with adversarial DA method to train a target-domain classifier. 

\begin{figure*}[!t]
	\centering
\subfigure[WA.]{
\label{fig:t_sne_a}
\includegraphics[width=0.3\textwidth]{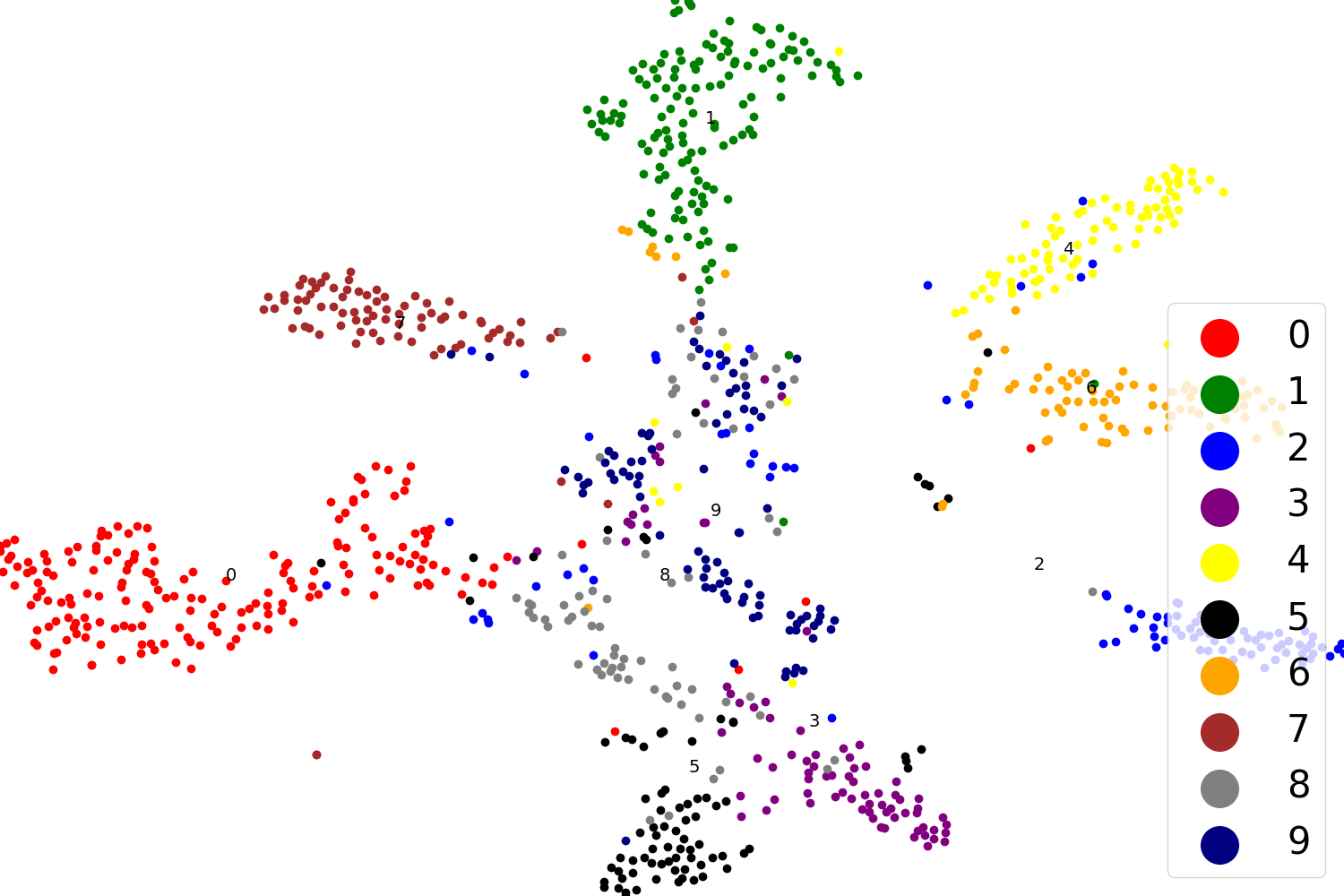}
}
\subfigure[FT.]{
\label{fig:t_sne_b}
\includegraphics[width=0.3\textwidth]{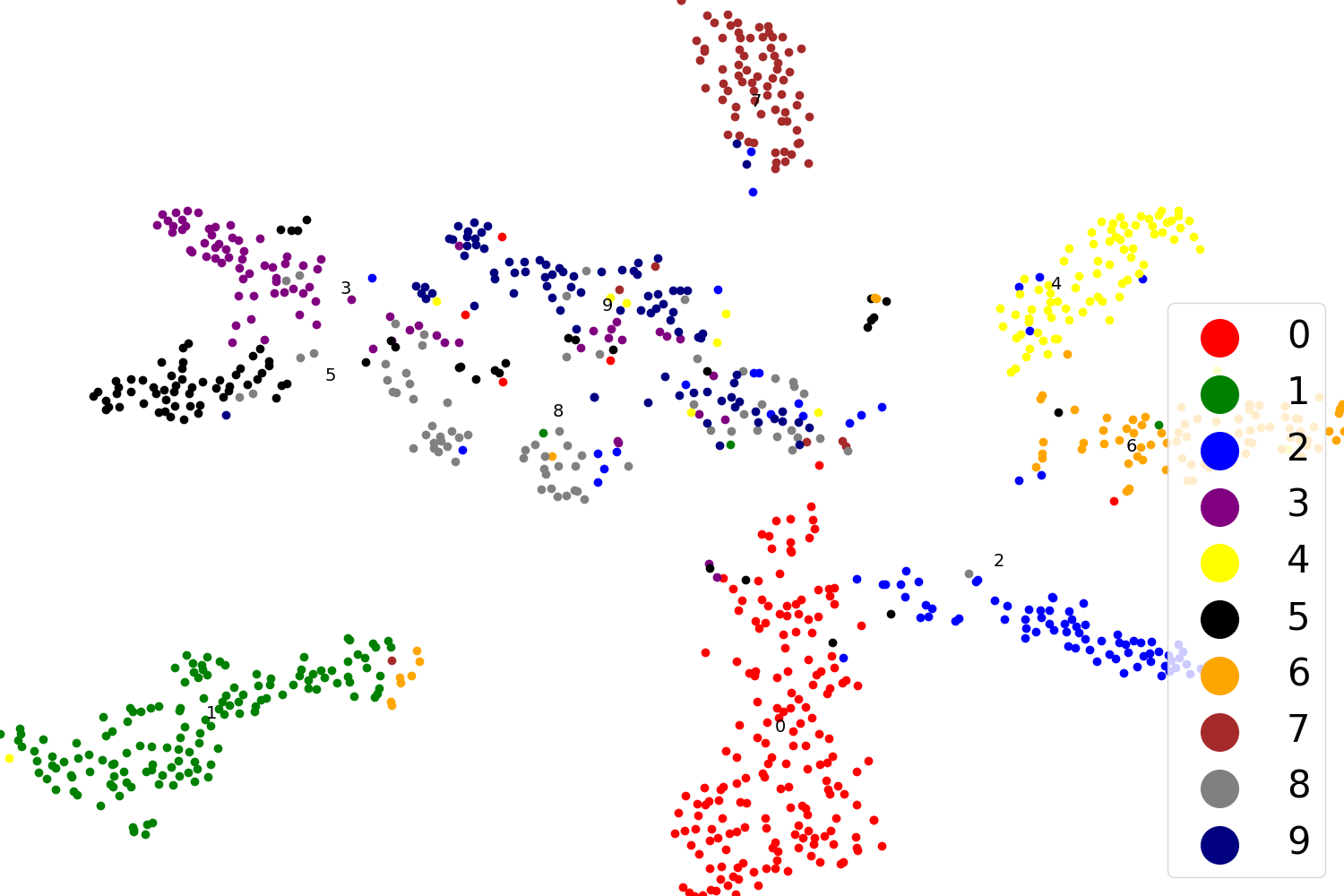}
}
%\quad
\subfigure[SHOT.]{
\label{fig:t_sne_c}
\includegraphics[width=0.3\textwidth]{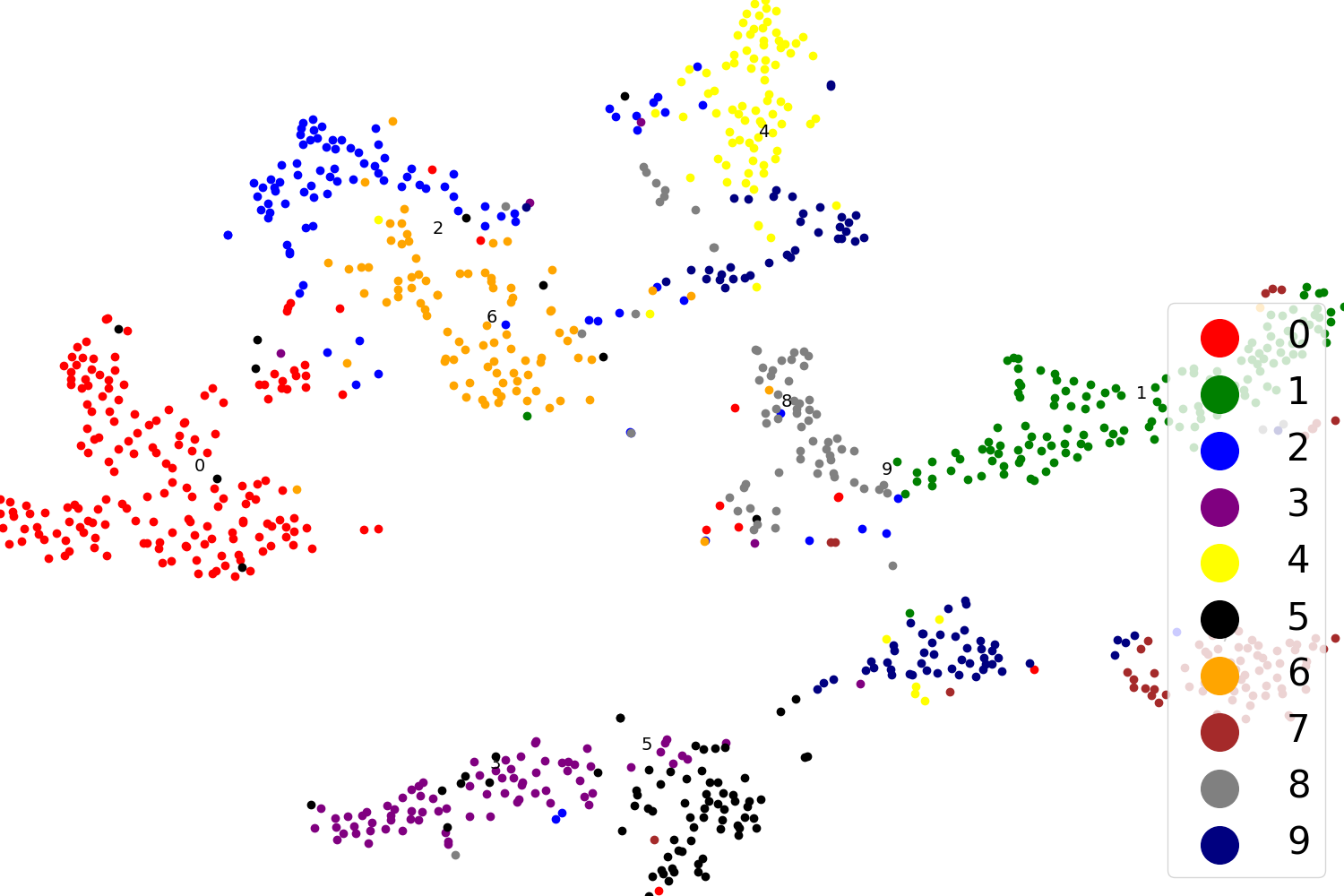}
}
\quad
% %\vspace{-1em}
\subfigure[S+F.]{
\label{fig:t_sne_d}
\includegraphics[width=0.3\textwidth]{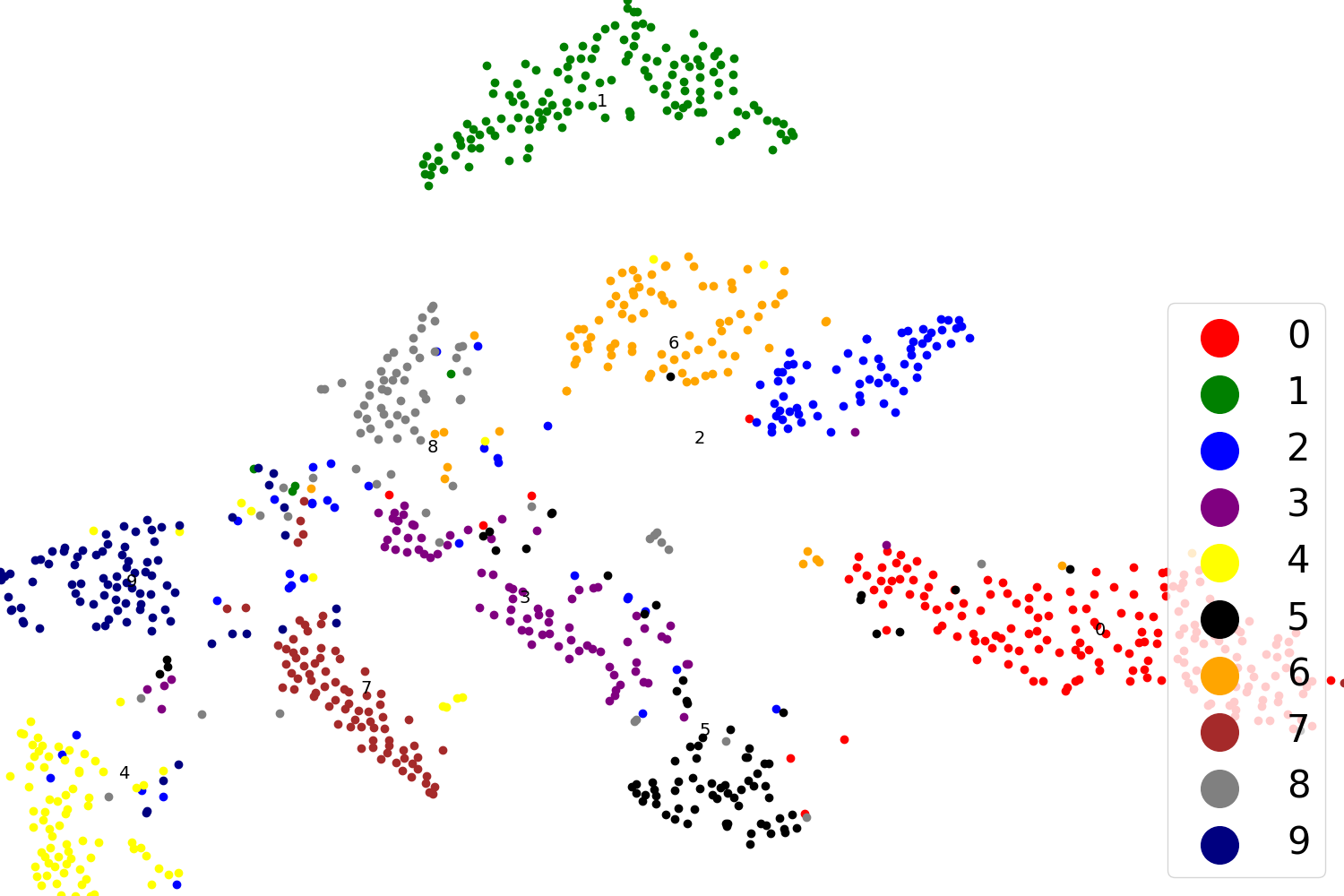}
}
\subfigure[T+F.]{
\label{fig:t_sne_e}
\includegraphics[width=0.3\textwidth]{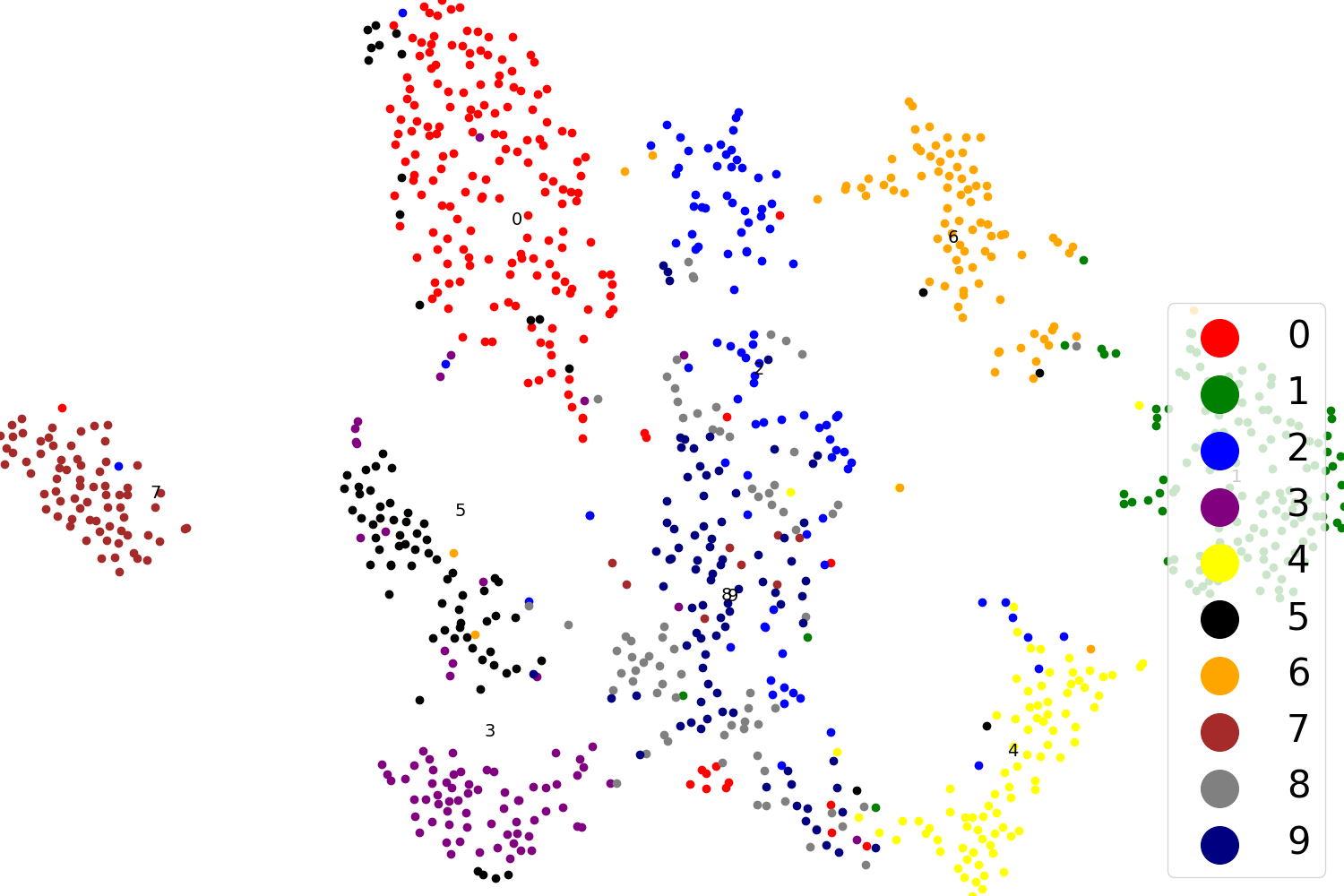}
}
\subfigure[TOHAN.]{
\label{fig:t_sne_f}
\includegraphics[width=0.3\textwidth]{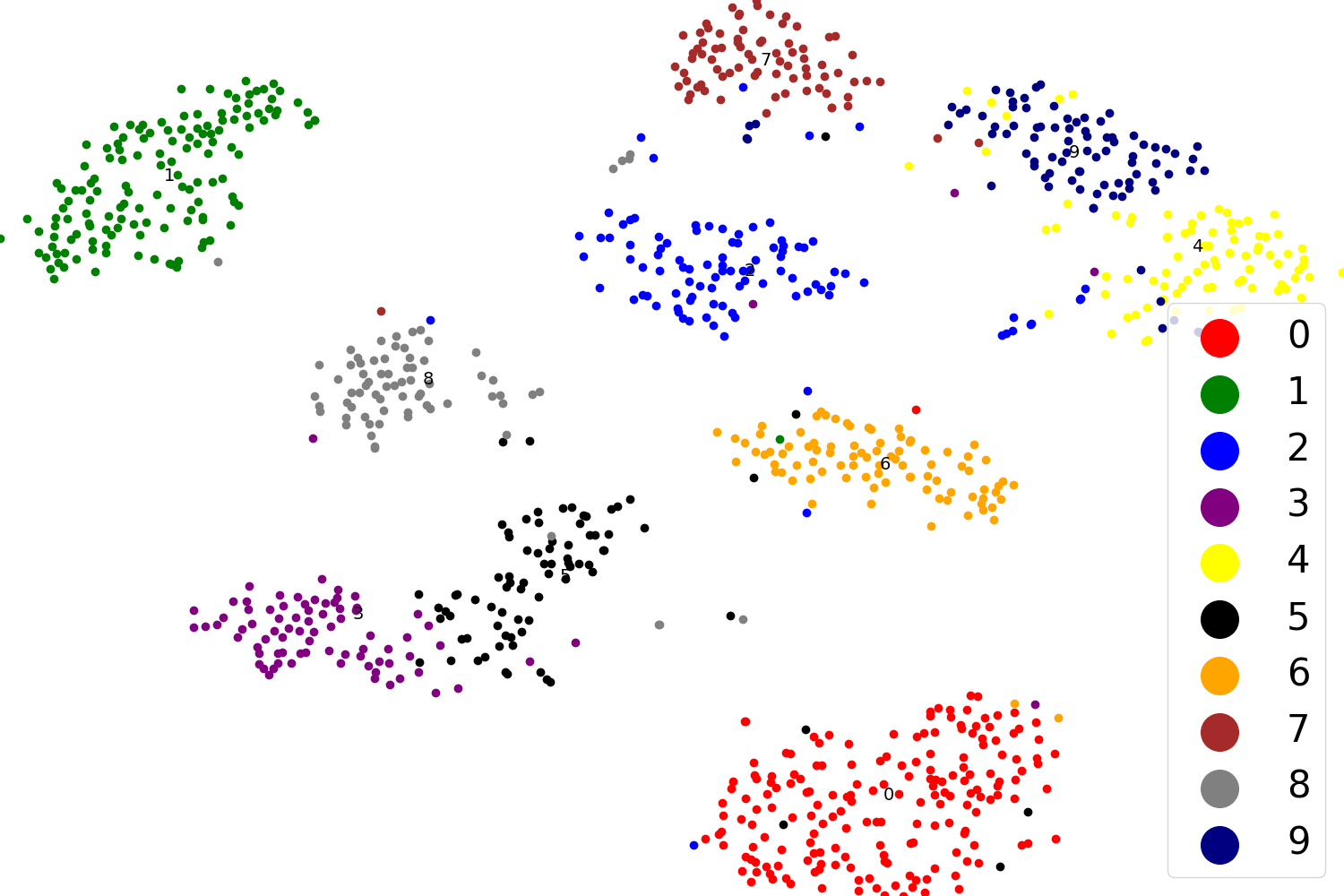}
}
	\caption{The t-SNE visualization for a 10-way classification task (taking \emph{MNIST}$\rightarrow$\emph{USPS} as an example). When we use WA and FT methods, nearly all classes mix together. Although the classification accuracy of SHOT, S+F and T+F are relatively high, there are still a little mixtures among classes. For TOHAN, it can be seen that all classes are separated well. Namely, TOHAN works well for solving FHA problem.}
	\label{fig:t-sne}
% 	\vspace{-1em}
\end{figure*}

\section{Implementation Details}
We implement all methods by PyTorch 1.7.1 and Python 3.7.6, and conduct all the experiments on two NVIDIA RTX 2080Ti GPUs.

\label{Asec:Details}
\paragraph{Network architecture.} We select architecture of generators \emph{$G_n$} ($n=1,\dots,N$) from DCGAN \citep{radford2015unsupervised}. For digits tasks, the encoder \emph{g}, classifier \emph{h} and group discriminator \emph{D} share the same architecture in all $6$ tasks, following FADA \citep{motiian2017few}. As for encoder \emph{g}, we employ the backbone network of LeNet-$5$ with batch normalization and dropout. For classifier \emph{h}, we adopt one fully connected layer with softmax fuction. For group discriminator \emph{D}, we adopt $3$ connected layers with softmax function. For objects tasks, we employ Densenet-$169$ \citep{huang2017densely} as encoder \emph{g}.

\paragraph{Network hyper-parameters.} We set fixed hyper-parameters in every method which is irrelevant to dataset, based on the common protocol of domain adaptation \citep{DBLP:conf/iclr/ShuBNE18}. The batch size of generator is set to $32$, and the batch size of group discriminator, encoder, classifier is all $64$. We pre-train the group discriminator for $100$ epochs. Meanwhile, the numbers of training steps of generator, group discriminator, encoder, classifier are set to $500$, $50$, $50$, $50$, respectively. Adam optimizer is with the same learning rate of $1\times10^{-3}$ in generators, encoder, classifier and group discriminator. The tradeoff parameter $\beta$ in Eq.~\eqref{eq:5} is set to $\frac{2}{1+\text{exp}(-10\dot q)}-1$, same as \cite{long2018conditional}. And the tradeoff parameter $\lambda$ in Eq.~\eqref{eq:3} is set to $0.2$ fixed. For the fair comparisons, we only resize and normalize the image and do not use any addition data augment or transformation. Note that, for each experiment, we report the result of the model trained in the last epoch.

\section{Additional Experiments about HTL}
\label{Asec:htl}
In this section, we compare TOHAN with another novel HTL method, i.e., \emph{dynamic knowledge distillation for HTL} (dkdHTL) \cite{yu2020dynamic}. It is worth noting that dkdHTL is a black-box HTL method. That is, we cannot access the parameters of source model. Therefore, in distillation loss, we cannot get the logits, which is used to compute the soften probabilities with a high temperature $T>1$. To address this problem, they tried to solve the logits through soften probabilities approximately. For the sake of fairness, we convert dkdHTL to white-box version. Specifically, we use the standard softmax function in distillation loss, instead of the approximate version. Moreover, we initial the parameters of target model by source model. Then, we show the results of dkdHTL in Table~\ref{dkdHTL:digit} and Table~\ref{dkdHTL:object}.

We find that TOHAN outperforms dkdHTL in most tasks significantly. However, in \emph{S}$\rightarrow$\emph{M}, \emph{U}$\rightarrow$\emph{M}, and \emph{S}$\rightarrow$\emph{U}, there exists few subtasks that dkdHTL outperforms TOHAN. In these three tasks, the complexity of source domain is high, while the complexity of target domain is low. TOHAN cannot generate qualified intermediate data effectively when the number of target data is very few and the source domain is highly complex simultaneously. dkdHTL is only suitable for tasks with uncomplicated target domain. The main reason is that the training data of dkdHTL are only the few target data. If the target domain is complex, it is very easy to overfit. Therefore, as for tasks with complex target domains, TOHAN has the upper hand.

\begin{table*}[!t]
\centering
\footnotesize
%\small
\setlength\tabcolsep{4.4pt}
\caption{Comparison between dkdHTL and TOHAN. We report the classification accuracy$\pm$standard deviation ($\%$) on $6$ digits FHA tasks. Bold value represents the highest accuracy on each column.}
\vspace{1mm}
% \resizebox{0.92\columnwidth}{!}{%
  \begin{tabular}{  l  c  c  c  c  c  c  c  c   c  c  c}
  \toprule
   \multicolumn{1}{c}{\multirow{2}{*}{Tasks}}  & 
   \multicolumn{1}{c}{FHA} & \multicolumn{7}{c}{{Number of Target Data per Class}} \\
   \cline{3-9}
 & Methods & {1} & {2}  & {3} & {4} & {5}  & {6} & {7} \\
\midrule
 \multirow{2}{*}{\emph{M}$\rightarrow$\emph{S}}  & {dkdHTL} & 24.1$\pm$0.7 & 24.1$\pm$0.3 & 24.5$\pm$0.6 & 24.4$\pm$1.1 & 25.4$\pm$0.8 & 25.7$\pm$0.5 & 26.1$\pm$1.1 \\
 &  {TOHAN} & \textbf{26.7$\pm$0.1} & \textbf{28.6$\pm$1.1} & \textbf{29.5$\pm$1.4} & \textbf{29.6$\pm$0.4} & \textbf{30.5$\pm$1.2} & \textbf{32.1$\pm$0.2} & \textbf{33.2$\pm$0.8} \\
\hline
 \multirow{2}{*}{\emph{S}$\rightarrow$\emph{M}} & {dkdHTL} & 71.2$\pm$1.2 & \textbf{83.4}$\pm$0.4 & \textbf{88.5$\pm$0.6} & \textbf{88.2$\pm$0.7} & \textbf{89.5$\pm$0.7} & \textbf{89.6$\pm$0.4} & \textbf{90.3$\pm$0.2} \\
 & {TOHAN} & \textbf{76.0$\pm$1.9} & 83.3$\pm$0.3 & 84.2$\pm$0.4 & 86.5$\pm$1.1 & 87.1$\pm$1.3 & 88.0$\pm$0.5 & 89.7$\pm$0.5 \\
\hline
 \multirow{2}{*}{\emph{M}$\rightarrow$\emph{U}} & {dkdHTL} & 65.2$\pm$0.6 & 70.5$\pm$1.3 & 74.4$\pm$0.6 & 77.8$\pm$0.6 & 78.6$\pm$0.9 & 78.8$\pm$1.1 & 79.0$\pm$1.3 \\
 & {TOHAN} & \textbf{87.7$\pm$0.7} & \textbf{88.3$\pm$0.5} & \textbf{88.5$\pm$1.2} & \textbf{89.3$\pm$0.9} & \textbf{89.4$\pm$0.8} & \textbf{90.0$\pm$1.0} & \textbf{90.4$\pm$1.2} \\
\hline
 \multirow{2}{*}{\emph{U}$\rightarrow$\emph{M}} & {dkdHTL} & 83.2$\pm$0.2 & \textbf{85.5$\pm$0.5} & \textbf{85.9$\pm$0.4} & 85.7$\pm$0.8 & 86.2$\pm$0.2 & 86.2$\pm$0.4 & 86.8$\pm$0.3 \\
 & {TOHAN} & \textbf{84.0$\pm$0.5} & 85.2$\pm$0.3 & 85.6$\pm$0.7 & \textbf{86.5$\pm$0.5} & \textbf{87.3$\pm$0.6} & \textbf{88.2$\pm$0.7} & \textbf{89.2$\pm$0.5} \\
\hline
 \multirow{2}{*}{\emph{S}$\rightarrow$\emph{U}} & {dkdHTL} & \textbf{76.3$\pm$0.3} & \textbf{77.6$\pm$0.5} & 78.9$\pm$0.4 & 79.5$\pm$0.4 & 80.2$\pm$0.5 & 80.7$\pm$0.4 & 82.1$\pm$0.4 \\
 &{TOHAN} & 75.8$\pm$0.9 & 76.8$\pm$1.2 & \textbf{79.4$\pm$0.9} & \textbf{80.2$\pm$0.6} & \textbf{80.5$\pm$1.4} &\textbf{81.1$\pm$1.1} & \textbf{82.6$\pm$1.9} \\
\hline
 \multirow{2}{*}{\emph{U}$\rightarrow$\emph{S}} & {dkdHTL} & 20.5$\pm$0.8 & 20.9$\pm$0.4 & 21.7$\pm$0.3 & 23.8$\pm$0.2 & 24.5$\pm$0.7 & 25.5$\pm$0.6 & 25.7$\pm$0.4 \\
  & {TOHAN} & \textbf{29.9$\pm$1.2} & \textbf{30.5$\pm$1.2} & \textbf{31.4$\pm$1.1} & \textbf{32.8$\pm$0.9} & \textbf{33.1$\pm$1.0} & \textbf{34.0$\pm$1.0} & \textbf{35.1$\pm$1.8} \\
\bottomrule
\end{tabular}
\label{dkdHTL:digit}
% \vspace{-1em}
\end{table*}

\begin{table}[!t]
  \centering
  \small
  \caption{Comparison between dkdHTL and TOHAN. We report the classification accuracy$\pm$standard deviation ($\%$) on 2 objects FHA tasks: CIFAR-$10\to$ STL-$10$ and STL-$10\to$ CIFAR-$10$. Bold value represents the highest accuracy ($\%$) among TOHAN and benchmark solutions.}
  \vspace{1mm}
  %\small
    \begin{tabular}{ccc}
    \toprule
       Tasks  & dkdHTL   & TOHAN \\
    \midrule
    CIFAR-$10\to$ STL-$10$ & 70.8$\pm$0.7 & \textbf{72.8$\pm$0.1} \\
    \midrule
    STL-$10\to$ CIFAR-$10$ &  52.4$\pm$0.5 & \textbf{56.6$\pm$0.3}\\
    \bottomrule
    \end{tabular}%
    
  \label{dkdHTL:object}%
%   \vspace{-2em}
\end{table}%

\section{Additional Analysis}
\label{Asec:AddAna}

\paragraph{Visualization of Results.}
We use t-SNE to visualize the feature (the penultimate layer of the classifier) extracted by TOHAN and $5$ benchmark solutions on \emph{M}$\rightarrow$\emph{U} task (see Figure~\ref{fig:t-sne}). When we use WA and FT methods, nearly all classes mix together. Although the classification accuracy of SHOT, S+F and T+F are relatively high, there are still a little mixtures among classes. For TOHAN, all classes are separated well, which demonstrates that TOHAN works well for solving FHA problem.
\paragraph{Detailed Analysis of Ablation Study.}
Table~\ref{full_ablation_study} shows the full results of ablation study. It is clear that TOHAN performs better than the corresponding two-step approach ST+FADA. However, when the number of target data is too small, ST+FADA may outperform TOHAN with a small probability. The reason for this abnormal phenomenon may be the limitation of target data. Although we use the technique of paring data, overfitting still occurs when data are scarce.

\begin{table}[!t]
	\centering
	\small
	\setlength\tabcolsep{5pt}
	% \resizebox{0.92\columnwidth}{!}{%
	\caption{Ablation Study. Bold value represents the highest accuracy ($\%$) on each column. Data behind '$\pm$' is the standard derivation.}
	\vspace{1mm}
	\begin{tabular}{  l  c  c  c  c  c  c  c  c  c  c }
		\toprule
		\multicolumn{1}{c}{\multirow{2}{*}{Tasks}} & \multicolumn{1}{c}{FHA} & \multicolumn{7}{c}{{Number of target data}} \\
		\cline{3-9}
		& Methods & {1} & {2}  & {3} & {4} & {5}  & {6} & {7} \\
		\midrule
		\multirow{4}{*}{\emph{M}$\rightarrow$\emph{S}}  
		& {S+FADA} & 25.6$\pm$1.3 & 27.7$\pm$0.5 & 27.8$\pm$0.7 & 28.2$\pm$1.3 & 28.4$\pm$1.4 & 29.0$\pm$1.0 & 29.6$\pm$1.9 \\
		 & {T+FADA}& 25.3$\pm$1.0 & 26.3$\pm$0.8 & 28.9$\pm$1.0 & 29.1$\pm$1.3 & 29.2$\pm$1.3 & 31.9$\pm$0.4 & 32.4$\pm$1.8 \\
		 & {ST+FADA} & 25.7$\pm$0.7 & 28.1$\pm$0.9 & 28.5$\pm$1.2 & 29.2$\pm$1.0 & 29.2$\pm$0.8 & 31.3$\pm$1.7 & 32.0$\pm$0.8 \\
		 & {TOHAN} & \textbf{26.7$\pm$0.1} & \textbf{28.6$\pm$1.1} & \textbf{29.5$\pm$1.4} & \textbf{29.6$\pm$0.4} & \textbf{30.5$\pm$1.2} & \textbf{32.1$\pm$0.23} & \textbf{33.2$\pm$0.8} \\
		\hline
		\multirow{4}{*}{\emph{S}$\rightarrow$\emph{M}}  
		 & {S+FADA} & 74.4$\pm$1.5 & 83.1$\pm$0.7 & 83.3$\pm$1.1 & 85.9$\pm$0.5 & 86.0$\pm$1.2 & 87.6$\pm$2.6 & 89.1$\pm$1.0 \\
		 & {T+FADA} & 74.2$\pm$1.8 & 81.6$\pm$4.0 & 83.4$\pm$0.8 & 82.0$\pm$2.3 & 86.2$\pm$0.7 & 87.2$\pm$0.8 & 88.2$\pm$0.6 \\
		 & {ST+FADA} & 74.3$\pm$1.2 & \textbf{83.7$\pm$1.0} & 83.8$\pm$0.8 & 85.8$\pm$0.6 & 86.0$\pm$0.9 & 87.7$\pm$0.8 & 89.0$\pm$0.6 \\
		 & {TOHAN} & \textbf{76.0$\pm$1.9} & 83.3$\pm$0.3 & \textbf{84.2$\pm$0.4} & \textbf{86.5$\pm$1.1} & \textbf{87.1$\pm$1.3} & \textbf{88.0$\pm$0.5} & \textbf{89.7$\pm$0.5} \\
		\hline
		\multirow{4}{*}{\emph{M}$\rightarrow$\emph{U}}   
		 & {S+FADA} & 83.7$\pm$0.9 & 86.0$\pm$0.4 & 86.1$\pm$1.1 & 86.5$\pm$0.8 & 86.8$\pm$1.4 & 87.0$\pm$0.6 & 87.2$\pm$0.8 \\
		 & {T+FADA} & 84.2$\pm$0.1 & 84.2$\pm$0.3 & 85.2$\pm$0.9 & 85.2$\pm$0.6 & 86.0$\pm$1.5 & 86.8$\pm$1.5 & 87.2$\pm$0.5 \\
		 & {ST+FADA} & 86.1$\pm$1.5 & 87.1$\pm$1.6 & 86.9$\pm$0.7 & 87.9$\pm$1.1 & 88.0$\pm$1.2 & 88.3$\pm$0.7 & 88.5$\pm$1.3 \\
		 & {TOHAN} & \textbf{87.7$\pm$0.7} & \textbf{88.3$\pm$0.5} & \textbf{88.5$\pm$1.2} & \textbf{89.3$\pm$0.9} & \textbf{89.4$\pm$0.8} & \textbf{90.0$\pm$1.0} & \textbf{90.4$\pm$1.2} \\
		\hline
		\multirow{4}{*}{\emph{U}$\rightarrow$\emph{M}}   
		 & {S+FADA} & 83.2$\pm$0.2 & 83.9$\pm$0.3 & 84.9$\pm$1.2 & 85.6$\pm$0.5 & 85.7$\pm$0.6 & 86.2$\pm$0.6 & 87.2$\pm$1.1 \\
		 & {T+FADA} & 82.9$\pm$0.7 & 83.9$\pm$0.2 & 84.7$\pm$0.8 & 85.4$\pm$0.6 & 85.6$\pm$0.7 & 86.3$\pm$0.9 & 86.6$\pm$0.7 \\
		 & {ST+FADA} & \textbf{84.0$\pm$0.7} & 84.2$\pm$0.5 & 85.3$\pm$1.0 & 85.6$\pm$1.2 & 86.7$\pm$1.0 & 86.5$\pm$0.5 & 88.0$\pm$1.0 \\
		 & {TOHAN} & 84.0$\pm$0.5 & \textbf{85.2$\pm$0.3} & \textbf{85.6$\pm$0.7} & \textbf{86.5$\pm$0.5} & \textbf{87.3$\pm$0.6} & \textbf{88.2$\pm$0.7} & \textbf{89.2$\pm$0.5} \\
		\hline
		\multirow{4}{*}{\emph{S}$\rightarrow$\emph{U}}   
		 & {S+FADA} & 72.2$\pm$1.4 & 73.6$\pm$1.4 & 74.7$\pm$1.4 & 76.2$\pm$1.3 & 77.2$\pm$1.7 & 77.8$\pm$3.0 & 79.7$\pm$1.9 \\
		 & {T+FADA} & 71.7$\pm$0.6 & 74.3$\pm$1.9 & 74.5$\pm$0.8 & 75.9$\pm$2.1 & 77.7$\pm$1.5 & 76.8$\pm$1.8 & 79.7$\pm$1.9 \\
		 & {ST+FADA} & 73.1$\pm$0.9 & 75.2$\pm$1.3 & 75.9$\pm$0.8 & 76.3$\pm$1.5 & 78.3$\pm$1.6 & 79.1$\pm$1.7 & 79.7$\pm$1.6 \\
		 & {TOHAN} & \textbf{75.8$\pm$0.9} & \textbf{76.8$\pm$1.2} & \textbf{79.4$\pm$0.9} & \textbf{80.2$\pm$0.6} & \textbf{80.5$\pm$1.4} & \textbf{81.1$\pm$1.1} & \textbf{82.6$\pm$1.9} \\
		\hline
		\multirow{4}{*}{\emph{U}$\rightarrow$\emph{S}}   
		 & {S+FADA} & 28.1$\pm$1.2 & 28.7$\pm$1.3 & 29.0$\pm$1.2 & 30.1$\pm$1.1 & 30.3$\pm$1.3 & 30.7$\pm$1.0 & 30.9$\pm$1.5 \\
		 & {T+FADA} & 27.5$\pm$1.4 & 27.9$\pm$0.9 & 28.4$\pm$1.3 & 29.4$\pm$1.8 & 29.5$\pm$0.7 & 30.2$\pm$1.0 & 30.4$\pm$1.7 \\
		 & {ST+FADA} & 28.1$\pm$1.3 & 28.9$\pm$0.7 & 29.2$\pm$1.5 & 29.8$\pm$1.2 & 31.0$\pm$0.9 & 31.2$\pm$0.9 & 33.2$\pm$1.7 \\
		 & {TOHAN} & \textbf{29.9$\pm$1.2} & \textbf{30.5$\pm$1.2} & \textbf{31.4$\pm$1.1} & \textbf{32.8$\pm$0.9} & \textbf{33.1$\pm$1.0} & \textbf{34.0$\pm$1.0} & \textbf{35.1$\pm$1.8} \\
		\bottomrule
	\end{tabular}
 
	\label{full_ablation_study}
% 	\vspace{-1em}
\end{table}

\section{Limitations}
\label{Asec:limitations}
The main limitation in this paper is that the run time of TOHAN is a little long. The main reason causing the long run time is the generation part of TOHAN. Specifically, the second term of Eq.~\eqref{eq:3} is time-consuming, as we need to calculate the distances between each intermediate data and each target data. We will optimize the generation part to overcome the time-consuming problem.

Although the generation part of TOHAN is a little time consuming, it solves the challenge of lacking source data in FHA efficiently. TOHAN can generate intermediate data containing the knowledge of source domain and target domain. Therefore, we not only adapt more useful source domain knowledge to target domain, but also prevent the privacy leakage of source domain. 

\section{Potential Negative Societal Impacts}
\label{Asec:impact}
The main potential negative societal impact in this paper is that TOHAN has a certain randomness. This is, TOHAN may not perform well consistently across various tasks. For example, TOHAN may fail to adapt knowledge between two domains that have a large discrepancy. Therefore, if TOHAN makes a mistake in a critical area, the consequences will be bad.

\section*{Checklist}

% %%% BEGIN INSTRUCTIONS %%%
% The checklist follows the references.  Please
% read the checklist guidelines carefully for information on how to answer these
% questions.  For each question, change the default \answerTODO{} to \answerYes{},
% \answerNo{}, or \answerNA{}.  You are strongly encouraged to include a {\bf
% justification to your answer}, either by referencing the appropriate section of
% your paper or providing a brief inline description.  For example:
% \begin{itemize}
%   \item Did you include the license to the code and datasets? \answerYes{See Section~\ref{gen_inst}.}
%   \item Did you include the license to the code and datasets? \answerNo{The code and the data are proprietary.}
%   \item Did you include the license to the code and datasets? \answerNA{}
% \end{itemize}
% Please do not modify the questions and only use the provided macros for your
% answers.  Note that the Checklist section does not count towards the page
% limit.  In your paper, please delete this instructions block and only keep the
% Checklist section heading above along with the questions/answers below.
% %%% END INSTRUCTIONS %%%

\begin{enumerate}

\item For all authors...
\begin{enumerate}
  \item Do the main claims made in the abstract and introduction accurately reflect the paper's contributions and scope?
    \answerYes{}
  \item Did you describe the limitations of your work? 
    \answerYes{Detailed limitations are in Appendix~\ref{Asec:limitations}.}
  \item Did you discuss any potential negative societal impacts of your work?
    \answerYes{Detailed potential negative societal impacts are in Appendix~\ref{Asec:impact}.}
  \item Have you read the ethics review guidelines and ensured that your paper conforms to them?
    \answerYes{}
\end{enumerate}

\item If you are including theoretical results...
\begin{enumerate}
  \item Did you state the full set of assumptions of all theoretical results?
    \answerYes{Please see Section~\ref{sec:SSL_moti}.}
	\item Did you include complete proofs of all theoretical results?
    \answerYes{Please see Appendix~\ref{Asec:Thm}.}
\end{enumerate}

\item If you ran experiments...
\begin{enumerate}
  \item Did you include the code, data, and instructions needed to reproduce the main experimental results (either in the supplemental material or as a URL)?
    \answerYes{}
  \item Did you specify all the training details (e.g., data splits, hyperparameters, how they were chosen)?
    \answerYes{Please see Appendix~\ref{Asec:Details}.}
	\item Did you report error bars (e.g., with respect to the random seed after running experiments multiple times)?
    \answerYes{We have reported the standard deviations for each results.}
	\item Did you include the total amount of compute and the type of resources used (e.g., type of GPUs, internal cluster, or cloud provider)?
    \answerYes{Please see Appendix~\ref{Asec:Details}.}
\end{enumerate}

\item If you are using existing assets (e.g., code, data, models) or curating/releasing new assets...
\begin{enumerate}
  \item If your work uses existing assets, did you cite the creators?
    \answerYes{}
  \item Did you mention the license of the assets?
    \answerNA{}
  \item Did you include any new assets either in the supplemental material or as a URL?
    \answerNo{}
  \item Did you discuss whether and how consent was obtained from people whose data you're using/curating?
    \answerNo{We use only standard datasets.}
  \item Did you discuss whether the data you are using/curating contains personally identifiable information or offensive content?
    \answerNo{We use only standard datasets.}
\end{enumerate}

\item If you used crowdsourcing or conducted research with human subjects...
\begin{enumerate}
  \item Did you include the full text of instructions given to participants and screenshots, if applicable?
  \answerNA{}
  \item Did you describe any potential participant risks, with links to Institutional Review Board (IRB) approvals, if applicable?
  \answerNA{}
  \item Did you include the estimated hourly wage paid to participants and the total amount spent on participant compensation?
  \answerNA{}
\end{enumerate}

\end{enumerate}

\newpage

\end{document}